\PassOptionsToPackage{breaklinks=true}{hyperref}
\documentclass{colt2017} 

\usepackage{bbm}

\DeclareMathOperator*{\minimize}{minimize}

\newcommand{\field}[1]{\mathbb{#1}}
\newcommand{\R}{\field{R}}

\newcommand{\bmtx}{\begin{bmatrix}}
\newcommand{\emtx}{\end{bmatrix}}
\newcommand{\bsmtx}{\left[ \begin{smallmatrix}} 
\newcommand{\esmtx}{\end{smallmatrix} \right]} 
\newcommand{\bmatarray}[1]{\left[\begin{array}{#1}}
\newcommand{\ematarray}{\end{array}\right]}

\title[Analysis of Stochastic Optimization Methods Using Jump System Theory]{A Unified Analysis of Stochastic Optimization Methods Using \\Jump System Theory and Quadratic Constraints}
\usepackage{times}

 \coltauthor{\Name{Bin Hu} \Email{bhu38@wisc.edu}\\
 \addr Wisconsin Institute for Discovery, University of Wisconsin,
       Madison, USA
 \AND
 \Name{Peter Seiler} \Email{seile017@umn.edu}\\
 \addr  Department of Aerospace Engineering and Mechanics, University of Minnesota,
       Minneapolis, USA
\AND
 \Name{Anders Rantzer} \Email{rantzer@control.lth.se}\\
 \addr  Department of Automatic Control, Lund University, Lund,  Sweden
 }

\begin{document}

\maketitle

\begin{abstract}
We develop a simple routine unifying the analysis of several important recently-developed stochastic optimization methods including SAGA, Finito, and stochastic dual coordinate ascent (SDCA).   First, we show an intrinsic connection between stochastic optimization methods and dynamic jump systems, and propose a general jump system model for stochastic optimization methods. Our proposed model recovers SAGA, SDCA, Finito, and SAG as special cases.  Then we combine jump system theory with several simple quadratic inequalities to derive sufficient conditions for convergence rate certifications of the proposed jump system model under various assumptions (with or without individual convexity, etc). The derived conditions are linear matrix inequalities (LMIs) whose sizes roughly scale with the size of the training set. We make use of the symmetry in the stochastic optimization methods and reduce these LMIs to some equivalent small LMIs whose sizes are at most $3\times 3$.
 We solve these small LMIs to provide analytical proofs of
new convergence rates for SAGA, Finito and SDCA (with or without individual convexity).  We also explain why our proposed LMI fails in analyzing SAG. We reveal a key difference between SAG and other methods, and briefly discuss how to extend our LMI analysis for SAG. 
An advantage of our approach is that the proposed analysis can be automated for a large class of stochastic methods under various assumptions (with or without individual convexity, etc).
\end{abstract}


\section{Introduction}

Convergence proofs for optimization methods are typically derived in a case-by-case manner. It is an important task to develop more unifying analysis which can be automatically generalized for complicated algorithms.
The aim of this paper is to develop a unified analysis routine for a class of recently-developed stochastic optimization methods used in empirical risk minimization.
Consider the following finite sum minimization
\begin{align}
\label{eq:minP}
\minimize_{x\in \R^p} \,\, g(x):=\frac{1}{n} \sum_{i=1}^n f_i(x)
\end{align}
where $g: \R^p\rightarrow \R$ is the objective function.  The framework of ~\eqref{eq:minP} is useful for empirical risk minimization problems, e.g. $\ell_2$-regularized logistic regression problems \citep{teo2007}.
%

A widely-used approach for solving~\eqref{eq:minP} is the stochastic gradient (SG) method \citep{robbins1951, Bottou2004}.
However, the SG method only linearly converges to some tolerance of the optimum of~\eqref{eq:minP} given a well-chosen constant stepsize.  If a diminishing stepsize is used, the SG method will converge to the optimum but at a sublinear rate.

More recently, a class of new stochastic optimization methods have been proposed based on the idea of gradient aggregation. These methods converge linearly to the optimum point while preserving the iteration cost of the SG method. This family of gradient aggregation methods include SAG \citep{Roux2012, Schmidt2013}, SAGA \citep{defazio2014}, Finito \citep{defazio2014finito}, SDCA \citep{shalev2013, Shai2016} and
SVRG \citep{johnson2013}. 
Existing linear rate bounds of SAG, SAGA, Finito, SDCA and SVRG are derived in a case-by-case manner.  
Moreover, the existing rate results for SAG, SAGA and Finito require the  individual convexity of $f_i$. 
It is beneficial to develop a unified analysis framework which can be used to justify the existing rate results and  obtain
new rate bounds under various conditions (with or without individual convexity, etc).

Recently, semidefinite programs have been used to certify the performance of deterministic optimization methods \citep{drori2014, kim2016, Lessard2014, nishihara2015, taylor2017}.  Specifically, \citet{Lessard2014} provides a general analysis for deterministic first-order optimization methods (full gradient method, Nesterov's method, heavy ball method, etc) by adapting the integral quadratic constraint (IQC) framework \citep{Megretski1997} from control theory. 
The key insight there is that the deterministic first-order methods can be viewed as interconnections of a linear time-invariant (LTI) dynamic system and a nonlinearity. Then quadratic inequalities can be used to characterize the nonlinearity and formulate LMI conditions.

In this paper, we present a unified analysis framework for a large class of stochastic optimization methods including SAGA, Finito and SDCA.
Our approach here is inspired by the work of \citet{Lessard2014}, and can be viewed as its stochastic extension.
In our paper, the key insight is that many stochastic first-order methods can be viewed as an interconnection of a linear jump system and a static nonlinearity.  
Notice that a linear jump system is described by a linear state space model whose state matrices are functions of a jump parameter sampled from a given distribution.
Since Lyapunov theory for jump systems has been well established in the controls field, we can incorporate quadratic constraints to obtain semidefinite programs for linear rate analysis of these stochastic optimization methods.  Our main contributions are summarized as follows.
\begin{enumerate}
\item We present a unified jump system perspective on SAG, SAGA, Finito and SDCA. Specifically, we propose a general jump system model which governs the dynamics of a large family of stochastic methods including SAG, SAGA, Finito and SDCA. 
\item We present a unified (and in some sense even automated) analysis framework for SAGA, Finito and SDCA using jump system perspectives and quadratic constraints.  LMI conditions for a large class of stochastic methods under various conditions (with or without individual convexity, etc) are derived using one technique, and then solved to provide rate certificates.
\item We  analytically solve the resultant LMIs  to prove linear rate bounds for SAGA, Finito, and SDCA under different assumptions on $g$ and $f_i$. Our results provide alternative proofs for many existing rate bounds. In addition, we prove that SAGA without individual convexity achieves an $\epsilon$-optimal iteration complexity $\tilde{\mathcal{O}}\left((\frac{L^2}{m^2}+n)\log(\frac{1}{\epsilon})\right)$. We also prove Finito without individual convexity achieves an $\epsilon$-optimal complexity of $\tilde{\mathcal{O}}\left(n\log(\frac{1}{\epsilon})\right)$ if $n\ge \frac{48L^2}{m^2}$.  
\item Our quadratic constraint approach reveals a key difference between SAG and other methods. Specifically, SAGA, SDCA, and Finito only require simple quadratic inequalities used in this paper while SAG further requires more advanced quadratic inequalities to decode convexity. For this reason, the analysis of SAG is more involved, and our proposed LMI fails in analyzing SAG. We briefly sketch how to extend our LMI analysis for SAG. The extension requires incorporating more advanced quadratic inequalities into the LMI formulations.
\end{enumerate}


 The main advantage of our framework  is its flexibility.
The existing analysis for SAG, SAGA, Finito and SDCA is derived in a case-by-case manner.   Our jump system framework  provides a unified routine for analysis of such methods. Our analysis is highly repeatable and even ``automated" in the sense that all LMI conditions are formulated using one technique and can be numerically solved to guide our analytical rate proof constructions.
We emphasize that we view our LMI-based method as a complement rather than a replacement for existing proof techniques.  One can always solve our proposed LMIs numerically and use the numerical results to narrow down possible Lyapunov function structures and useful function inequalities even before trying to construct proofs. This complements several existing proof techniques which more or less require  guessing the required Lyapunov functions at the early stage of proof constructions.
 We will further explain this point after our main LMI condition is presented.

The rest of the paper is organized as follows. Section~\ref{sec:back} introduces the notation and reviews the concepts of linear jump systems.
 In Section~\ref{sec:ASAG}, we present a general jump system model which governs the dynamics of  a large family of stochastic optimization methods including SAG, SAGA, Finito and SDCA.    Section \ref{sec:IQCmain} presents a unified LMI analysis for the proposed jump system model. A unified LMI condition is derived using jump system theory and several function properties in the form of simple quadratic constraints.  We apply the LMI condition and successfully prove various rate bounds for SAGA, SDCA, and Finito with or without individual convexity. We also explain why our proposed LMI fails in analyzing SAG. We reveal a key difference between SAG and other methods, and briefly discuss how to extend our LMI analysis for SAG. 
 We present the main technical proofs in Section \ref{sec:MainProof}. 
   Finally, we conclude with several future directions (Section~\ref{sec:ConR}).

\section{Preliminaries}
\label{sec:back}
\subsection{Notation and Background}


The set of $p$-dimensional real vectors is denoted as $\R^p$. 
The $p\times p$ identity matrix and the $p \times p$ zero matrix are denoted as $I_p$ and $0_p$, respectively. The $n\times n$ identity matrix is denoted as $I_n$, and the $n\times n$ zero matrix is denoted as $0_n$. Let $e_i$ denote the $n$-dimensional vector whose entries are all $0$ except the $i$-th entry which is $1$. Let $e$ denote the $n$-dimensional vector whose entries are all $1$. Let $\tilde{0}$ denote the $n$-dimensional vector whose entries are all $0$. For simplicity,  $0$ is occasionally used to denote a zero vector or a zero matrix when there is no confusion on the dimension.
The Kronecker product of two matrices $A$ and $B$ is denoted by $A \otimes B$. Notice $(A\otimes B)^T=A^T \otimes B^T$ and $(A\otimes B)(C\otimes D)=(AC)\otimes (BD)$ when the matrices have compatible dimensions. When a matrix $P$ is negative semidefinite (definite), we will use the notation $P\le (<) 0$. When $P$ is positive definite, we use the notation $P>0$.

A continuously differentiable function $f:\R^p\rightarrow \R$ is $L$-smooth if  for all $x, y\in \R^p$ we have
$\|\nabla f(x)-\nabla f(y)\|\le L \|x-y\|$.
The continuously differentiable function $f$ is said to be $m$-strongly convex if  for all $x,y \in \R^p$ we have
$f(x)\ge f(y)+\nabla f(y)^T (x-y)+\frac{m}{2} \|x-y\|^2$.
Notice $f$ is said to be convex if  $f$ is $0$-strongly convex.
Let $\mathcal{F}(m,L)$ denote the set of continuously differentiable
functions $f: \R^p \rightarrow \R$  that are $L$-smooth and $m$-strongly convex. Hence $\mathcal{F}(0,L)$ denotes the set of continuously differentiable convex
functions that are $L$-smooth.

For any $f \in \mathcal{F}(m,L)$ with $m>0$, there exist a unique $x^*\in \R^p$ such that $\nabla f(x^*)=0$. In addition, the following inequality holds for any $x\in \R^p$ \cite[Proposition 5]{Lessard2014}
\begin{align}
\label{eq:gradient3}
\bmtx x-x^* \\ \nabla f(x)\emtx^T \bmtx -2mL I_p & (L+m) I_p \\  (L+m) I_p & -2 I_p\emtx  
\bmtx x-x^* \\ \nabla f(x)\emtx\ge 0
\end{align}

However, a function satisfying the above inequality may not belong to $\mathcal{F}(m,L)$, and may not even be convex. The set of continuously differentiable functions satisfying \eqref{eq:gradient3} with some unique global minimum $x^*$  is denoted as $\mathcal{S}(m,L)$. This class of functions has sector-bounded gradients, and includes $\mathcal{F}(m,L)$ as its subset. We emphasize that the functions in $\mathcal{S}(m,L)$ may not be convex. 

A general assumption adopted in this paper is that $g\in \mathcal{S}(m, L)$ with $m>0$. This is weaker than the assumption $g\in \mathcal{F}(m,L)$. Three sets of assumptions are typically used for $f_i$, i.e. $f_i \in \mathcal{F}(m, L)$, $f_i \in \mathcal{F}(0, L)$ or $f_i$ being $L$-smooth.
Given an arbitrary reference point $x^*$ (the value of $\nabla f_i(x^*)$ may not be $0$) and any $x\in \R^p$, the following inequality always holds 
\begin{align}
\label{eq:gradientfi}
\bmtx x-x^* \\ \nabla f_i(x)-\nabla f_i(x^*)\emtx^T \bmtx 2L\gamma I_p & (L-\gamma) I_p \\  (L-\gamma) I_p & -2 I_p\emtx  
\bmtx x-x^* \\ \nabla f_i(x)-\nabla f_i(x^*)\emtx\ge 0
\end{align}
where $\gamma$ is determined by the assumptions on $f_i$ as follows
\begin{align}
  \label{eq:gammav}
  \gamma :=  \left\{
    \begin{array}{ll}
      -m & \mbox{if } f_i\in \mathcal{F}(m,L)\\
      0 & \mbox{if } f_i\in \mathcal{F}(0, L)\\
     L & \mbox{if } f_i \mbox{ is } L\mbox{-smooth} 
    \end{array}
  \right..
\end{align}
Notice \eqref{eq:gradientfi} is just a summary of the definition of $L$-smoothness and the so-called co-coercivity condition \citep[Proposition 5]{Lessard2014}.

Finally, the underlying probability space for the sampling index $i_k$ is denoted as $(\Omega, \mathcal{F}, \mathbb{P})$. Let $\mathcal{F}_k$ be the $\sigma$-algebra generated by $(i_1, i_2, \ldots, i_k)$. Clearly, $i_k$ is $\mathcal{F}_k$-adapted and we obtain a filtered probability space $(\Omega, \mathcal{F}, \{\mathcal{F}_k\}, \mathbb{P})$ which the stochastic method is defined on.

\subsection{Stochastic Jump Systems}
\label{sec:CJS}

A linear jump system is described by the following set of recursive equations:
\begin{align}
\begin{split}
\label{eq:jump1}
\xi^{k+1}&=A_{i_k}\xi^k+B_{i_k}w^k\\
v^k&=C_{i_k} \xi^k+D_{i_k}w^k.
\end{split}
\end{align}
At each step $k$, the jump parameter $i_k$ is a random variable taking value in a finite set $\mathcal{N}=\{1,\cdots,n\}$.
In addition, $A_{i_k}:\mathcal{N}\rightarrow \R^{n_\xi\times n_\xi}$,  $B_{i_k}:\mathcal{N}\rightarrow \R^{n_\xi\times n_w}$, $C_{i_k}:\mathcal{N}\rightarrow \R^{n_v\times n_\xi}$, and $D_{i_k}:\mathcal{N}\rightarrow \R^{n_v\times n_w}$ are functions of $i_k$.
When $i_k=i\in \mathcal{N}$, clearly we have $A_{i_k}=A_i$, $B_{i_k}=B_i$, $C_{i_k}=C_i$, and $D_{i_k}=D_i$. 
If the process $\{i_k: k=1,2,\ldots\}$ is a Markov chain, the resultant jump system \eqref{eq:jump1} is termed as a discrete-time Markovian jump linear system (MJLS). There is a large body of literature on MJLS in the controls field \citep{costa2006, dragan2010}. We confine our scope to the special case where $i_k$ is an identically and independently distributed (IID) process, i.e. $\mathbb{P}(i_k=i \, \vert \, \mathcal{F}_{k-1})=\mathbb{P}(i_k=i)$ for all $k\ge 0$ and $i\in \mathcal{N}$.
When $i_k$ is sampled from a uniform distribution, we have $\mathbb{P}(i_k=i)=\frac{1}{n}$.
When $i_k$ is generated cyclically based on a deterministic order, \eqref{eq:jump1} is not a jump system but a linear periodic system. There is also a large body of control literature on linear periodic systems \citep{bittanti2008}. When $i_k$ is a constant, then the state matrices are constant matrices and the model \eqref{eq:jump1} is just an LTI system. LTI system theory is also well established ~\citep{hespanha2009}.

\section{A General Jump System Model for Stochastic Optimization Methods}
\label{sec:ASAG}
Now we introduce the following general jump system model which governs the dynamics of a large family of stochastic optimization methods.
\begin{align}
\label{eq:spfdJump0}
\begin{split}
\xi^{k+1}&=A_{i_k} \xi^k +B_{i_k} w^k\\
v^k&=C \xi^k\\
w^k&=\bmtx \nabla f_1(v^k) \\ \nabla f_2(v^k) \\ \vdots \\ \nabla f_n(v^k)\emtx
\end{split}
\end{align}
The above model builds upon the linear jump system model \eqref{eq:jump1} by further enforcing a nonlinear relationship between $w^k$ and $v^k$, i.e.  $w^k=\bmtx \nabla f_1(v^k)^T \cdots \nabla f_n(v^k)^T \emtx^T$.
We can represent a large family of stochastic optimization methods using the unified jump system model \eqref{eq:spfdJump0} with properly chosen $(A_{i_k}, B_{i_k}, C)$. In this paper, we consider the following stochastic methods.

\begin{enumerate}
\item SAGA \citep{defazio2014}: The iteration rule is the follows
\begin{align}
\label{eq:SAGAx}
x^{k+1}=x^k-\alpha\left(\nabla f_{i_k}(x^k)-y_{i_k}^{k}+\frac{1}{n}\sum_{i=1}^n y_i^{k}\right)
\end{align}
where at each step $k$, a random training example $i_k$ is drawn uniformly from the set $\mathcal{N}$ and 
\begin{align}
  \label{eq:SAGy}
  y_i^{k+1} :=  \left\{
    \begin{array}{ll}
      \nabla f_i(x^k) & \mbox{if } i=i_k\\
      y_i^{k} & \mbox{otherwise}
    \end{array}
  \right..
\end{align}
\item  SAG \citep{Roux2012, Schmidt2013}: The main iteration rule is 
\begin{align}
\label{eq:SAGx}
x^{k+1}=x^k-\alpha\left(\frac{\nabla f_{i_k}(x^k)-y_{i_k}^k}{n}+\frac{1}{n}\sum_{i=1}^n y_i^k\right)
\end{align} 
where at each $k$, $i_k$ is uniformly drawn from the set $\mathcal{N}$ and $y_i^k$ is updated by~\eqref{eq:SAGy}.
\item Finito \citep{defazio2014finito}:  Suppose $x_i^k\in \R^p$ and $y_i^k\in \R^p$ for each $k$ and all $i\in \mathcal{N}$.  At each $k$,  an index $i_k$ is drawn from the set $\mathcal{N}$, and $x_i^{k+1}$ is updated as
\begin{align}
  \label{eq:Finitox}
  x_i^{k+1} :=  \left\{
    \begin{array}{ll}
    \frac{1}{n}\sum_{i=1}^n x_i^k-\alpha\sum_{i=1}^n  y_i^k   & \mbox{if } i=i_k\\
      x_i^{k} & \mbox{otherwise}
    \end{array}
  \right..
\end{align}
where $\alpha$ is the stepsize \footnote{One typical choice of $\alpha$ under the big data condition is $\alpha=\frac{1}{2nm}$.}. Then $y_i^{k+1}$ is updated as
\begin{align}
  \label{eq:Finitoy}
  y_i^{k+1} :=  \left\{
    \begin{array}{ll}
      \nabla f_i(x_i^{k+1}) & \mbox{if } i=i_k\\
      y_i^{k} & \mbox{otherwise}
    \end{array}
  \right..
\end{align}
\item SDCA \citep[Algorithm 1]{Shai2016}: There are several versions of SDCA. For simplicity, we consider SDCA without duality, which solves the $\ell_2$-regularized problem
\begin{align}
\label{eq:minPl2}
\minimize_{x\in \R^p} \,\, g(x):=\frac{1}{n} \sum_{i=1}^n f_i(x)+\frac{m}{2}\|x\|^2
\end{align}
To solve the above problem, SDCA without duality requires updating $x^{k}\in \R^p$ and $y_i^{k+1}\in \R^p$ at each step. It first updates $x^{k}$ using $y_i^k$ as follows
\begin{align}
\label{eq:SDCAx}
x^k=\frac{1}{m n}\sum_{i=1}^n y_i^k
\end{align}
Then $y_i^{k+1}$ is updated as
 \begin{align}
  \label{eq:SDCAy}
  y_i^{k+1} :=  \left\{
    \begin{array}{ll}
      y_i^k-\alpha m n(\nabla f_i(x^k)+y_i^k) & \mbox{if } i=i_k\\
      y_i^{k} & \mbox{otherwise}
    \end{array}
  \right..
\end{align}
where $i_k$ is randomly sampled from $\mathcal{N}$. In the actual computation, the update \eqref{eq:SDCAx} for $k\ge 1$ is performed using the formula  $x^{k}=x^{k-1}-\alpha (\nabla f_{i_{k-1}}(x^{k-1})+y_i^{k-1})$ due to efficiency considerations. However,  \eqref{eq:SDCAx} is more general and governs the updates of SDCA for all $k$.
\end{enumerate}

To represent the above methods in the general jump system model \eqref{eq:spfdJump0}, we can choose the state matrices as $A_{i_k}=\tilde{A}_{i_k}\otimes I_p$, $B_{i_k}=\tilde{B}_{i_k}\otimes I_p$, and $C=\tilde{C}\otimes I_p$ where $\tilde{A}_{i_k}$, $\tilde{B}_{i_k}$ and $\tilde{C}$ are defined according to Table \ref{tab:JumpModel}. 
\begin{table}[h]
\resizebox{1\textwidth}{!}{\begin{minipage}{\textwidth}
\centering
\begin{tabular}{c|c|c|c}
\hline
 Method & $\tilde{A}_{i_k}$ &  $\tilde{B}_{i_k}$ & $\tilde{C}$ \\ \hline\rule{0pt}{8mm}
SAGA & $\bmtx I_n-e_{i_k} e_{i_k}^T & \tilde{0}  \\ -\frac{\alpha}{n}(e-ne_{i_k})^T & 1 \emtx$ & $\bmtx e_{i_k} e_{i_k}^T \\[1mm] -\alpha e_{i_k}^T \emtx$ & $\bmtx \tilde{0}^T & 1 \emtx$  \\[5mm]
 SAG & $\bmtx I_n-e_{i_k} e_{i_k}^T & \tilde{0}  \\ -\frac{\alpha}{n}(e-e_{i_k})^T & 1 \emtx$ & $\bmtx e_{i_k} e_{i_k}^T \\[1mm] -\frac{\alpha}{n}e_{i_k}^T \emtx$ & $\bmtx \tilde{0}^T & 1 \emtx$ \\[5mm]
Finito & $\bmtx I_n- e_{i_k} e_{i_k}^T & \tilde{0}  \\  -\alpha (e_{i_k}e^T) &  I_n- e_{i_k} (e_{i_k}^T-\frac{1}{n}e^T) \emtx$ & $\bmtx e_{i_k} e_{i_k}^T \\ \tilde{0}\tilde{0}^T \emtx$ & $\bmtx  -\alpha e^T  & \frac{1}{n}e^T \emtx$  \\[5mm]
SDCA  & $I_n-\alpha m n e_{i_k} e_{i_k}^T$ & $-\alpha m n e_{i_k} e_{i_k}^T$ & $\frac{1}{mn}e^T$  \\[2mm]  \hline
\end{tabular}
\vspace{0.1in}
\caption{State Matrices for Feedback Representations of SAG, SAGA, Finito, and SDCA}
\label{tab:JumpModel}
\end{minipage} }
\end{table}

For illustrative purposes, we explain the jump system formulation for SAGA. The jump system formulations for SAG, Finito, and SDCA are further explained in Appendix~\ref{sec:JumpF}. For SAGA, we define the stacked vector 
$y^k:=\bmtx (y_1^k)^T  & \cdots  & (y_n^k)^T \emtx^T$.
Then the SAGA gradient update rule \eqref{eq:SAGy} can be rewritten as:
\begin{align}
\begin{split}
\label{eq:SAGyJump}
y^{k+1}=\left((I_n-e_{i_k} e_{i_k}^T)\otimes I_p\right) y^{k}+ \left((e_{i_k} e_{i_k}^T)\otimes I_p\right) w^k
\end{split}
\end{align}
where $w^k=\bmtx \nabla f_1(x^k)^T \cdots \nabla f_n(x^k)^T \emtx^T$.
Notice $\sum_{i=1}^n y_i^{k} = ( e^T \otimes I_p) y^{k}$ and $\nabla f_{i_k}(x^k)-y_{i_k}^k=(e_{i_k}^T\otimes I_p)(w^k-y^k)$.
Thus the iteration rule \eqref{eq:SAGAx} can be rewritten as follows:
\begin{align}
\label{eq:SAGAxJump}
\begin{split}
x^{k+1}&=x^k-\alpha (e_{i_k}^T \otimes I_p)(w^k-y^{k})-\frac{\alpha}{n}(e^T\otimes I_p) y^{k}\\
&=x^k-\frac{\alpha}{n}\left((e-ne_{i_k})^T\otimes I_p\right)y^{k}-\alpha(e_{i_k}^T\otimes I_p) w^k
\end{split}
\end{align}
Now the update rules in \eqref{eq:SAGyJump} and \eqref{eq:SAGAxJump} can be expressed as:
\begin{align}
\label{eq:SAGAG}
\begin{split}
\bmtx y^{k+1} \\ x^{k+1}\emtx&=\bmtx (I_n-e_{i_k} e_{i_k}^T)\otimes I_p & \tilde{0} \otimes I_p \\ -\frac{\alpha}{n}(e-ne_{i_k})^T\otimes I_p & I_p \emtx \bmtx y^{k} \\ x^k \emtx+\bmtx(e_{i_k} e_{i_k}^T)\otimes I_p \\ (-\alpha e_{i_k}^T)\otimes I_p \emtx w^k\\
v^k&=\bmtx \tilde{0}^T \otimes I_p& I_p \emtx \bmtx y^{k}\\ x^k\emtx\\
w^k&=\bmtx \nabla f_1(v^k) \\ \vdots \\ \nabla f_n(v^k) \emtx
\end{split}
\end{align}
which is exactly in the form of the general jump system model \eqref{eq:spfdJump0} with $\xi^k = \bsmtx y^k \\  x^k\esmtx$.  

The computation of $w^k$ at each $k$ requires a full gradient computation (or $n$ individual oracle accesses). However, $B_{i_k}$ is sparse  such that $B_{i_k}w^k$ only involves one individual oracle access.  The low per-iteration cost of stochastic methods is captured by the sparsity of $B_{i_k}$. Most entries of $w^k$ are ``phantom" iterates which facilitates our analysis but do not appear in the actual computation.

Since $g\in \mathcal{S}(m,L)$ with $m>0$, there exists unique $x^*\in \R^p$ satisfying $\nabla g(x^*)=0$. To make \eqref{eq:spfdJump0} a good model for optimization methods, we have to ensure its equilibrium point is related to $x^*$. Define $w^*:=\bmtx \nabla f_1(x^*)^T &\ldots &\nabla f_n(x^*)^T\emtx^T$,  and $v^*:=x^*$. If \eqref{eq:spfdJump0} is an optimization method which converges to $x^*$, then $\xi^k$ should converge to some equilibrium state $\xi^*$ capturing the information of $x^*$ and satisfying
\begin{align}
\label{eq:spfdJumpEqi}
\begin{split}
\xi^*&=A_{i} \xi^* +B_{i} w^*\\
v^*&=C \xi^*\\
w^*&=\bmtx \nabla f_1(v^*)  \\ \vdots \\ \nabla f_n(v^*)\emtx
\end{split}
\end{align}
for all $i\in \mathcal{N}$. Now we set up $\xi^*$ for SAGA, SAG, Finito, and SDCA as follows.

\begin{enumerate}
\item  For SAG and SAGA, we have $\xi^k : = \bsmtx y^k \\  x^k\esmtx$ and $\xi^* : = \bsmtx w^* \\x^*\esmtx$. If we can show that $\xi^k$ converges to $\xi^*$, then we can conclude that $x^k$ converges to $x^*$ and $y_i^k$ converges to $\nabla f_i(x^*)$.
\item For Finito, we have $x^k:=\bmtx (x_1^k)^T  & \cdots  & (x_n^k)^T \emtx^T$, $\xi^*=\bmtx y^k \\ x^k \emtx$ and $\xi^*=\bmtx w^* \\ e \otimes x^* \emtx$. If we can show that $\xi^k$ converges to $\xi^*$, then $y_i^k$ converges to $\nabla f_i(x^*)$ and $x_i^k$ converges to $x^*$.
\item For SDCA (without duality), we have $\xi^k=y^k$ and $\xi^*=-w^*$. For the $\ell_2$-regularized problem \eqref{eq:minPl2} with strongly-convex $g$, the optimal point $x^*$  satisfies
$mx^*+\frac{1}{n}\sum_{i=1}^n \nabla f_i(x^*)=0$.  Hence, if $\xi^k$ converges to $\xi^*$, then $y_i^k$ converges to $-\nabla f_i(x^*)$ and $x^k$ converges to $x^*$.
\end{enumerate}
It is straightforward to verify that \eqref{eq:spfdJumpEqi} holds for the above $\xi^*$ due to the fact $\nabla g(x^*)=0$.

\section{Analysis of Stochastic Methods Using Semidefinite Programs}
\label{sec:IQCmain}
\subsection{An Unified LMI Condition for Analysis of Stochastic Methods}
\label{sec:IQCSAG}

%


From the above discussion, we always want to show $\xi^k$ converges to $\xi^*$ at a given linear rate $\rho$. Now
we present a unified LMI condition for such linear convergence using jump system theory and the basic quadratic inequalities \eqref{eq:gradient3} \eqref{eq:gradientfi} which capture the key properties of the loss functions.

\begin{theorem}
\label{thm:lmiS0}
Consider the general jump system model~\eqref{eq:spfdJump0}, where $A_i=\tilde{A}_i\otimes I_p$, $B_i=\tilde{B}_i \otimes I_p$, and $C=\tilde{C}\otimes I_p$.  Assume $i_k$ is sampled in an IID manner from a uniform distribution $\mathbb{P}(i_k=i)=\frac{1}{n}$. Suppose there exists a unique $x^*\in \R^p$ such that $\nabla g(x^*)=0$. 
The function $f_i$ is assumed to satisfy the following two inequalities for any $x\in \R^p$, 
\begin{align}
\label{eq:basicIq1}
&\bmtx x-x^* \\ \frac{\sum_{i=1}^n \nabla f_i(x)}{n}-\frac{\sum_{i=1}^n \nabla f_i(x^*)}{n}\emtx^T \bmtx 2L\nu I_p & (L-\nu) I_p \\  (L-\nu) I_p & -2 I_p\emtx  
\bmtx x-x^* \\ \frac{\sum_{i=1}^n \nabla f_i(x)}{n}-\frac{\sum_{i=1}^n \nabla f_i(x^*)}{n} \emtx\ge 0\\
\label{eq:basicIq2}
&\bmtx x-x^* \\ \nabla f_i(x)-f_i(x^*)\emtx^T \bmtx 2L\gamma I_p & (L-\gamma) I_p \\  (L-\gamma) I_p & -2 I_p\emtx  
\bmtx x-x^* \\ \nabla f_i(x)-f_i(x^*)\emtx\ge 0
\end{align}
where $\nu$ and $\gamma$ are some prescribed scalars. 
Define $\tilde{D}_{\psi 1}\in \R^{2n+2}$ and $\tilde{D}_{\psi 2}\in \R^{(2n+2)\times n}$  as 
\begin{align}
\begin{split}
 \tilde{D}_{\psi 1}&=\bmtx   L  & \nu  & L & \gamma & \ldots &  L &  \gamma \emtx^T \\
 \tilde{D}_{\psi 2}&=\bmtx     -\frac{1}{n}e  & \frac{1}{n}e &  -e_1 & e_1 & \ldots & -e_n & e_n \emtx^T.
\end{split}
\end{align}

 If $\exists$ an $n_\xi\times n_\xi$ matrix
$\tilde{P}=\tilde{P}^T > 0$ and nonnegative scalars $\lambda_1$, $\lambda_2$ such that 
\begin{align}
  \label{eq:lmiIQCS0}
\begin{split}
  \bmtx \frac{1}{n}\sum_{i=1}^n  \tilde{A}_i^T \tilde{P} \tilde{A}_i-\rho^2 \tilde{P}  & \frac{1}{n}\sum_{i=1}^n \tilde{A}_i^T \tilde{P} \tilde{B}_i \\[2mm] 
  \frac{1}{n}\sum_{i=1}^n \tilde{B}_i^T \tilde{P} \tilde{A}_i & \frac{1}{n}\sum_{i=1}^n\tilde{B}_i^T\tilde{P}\tilde{B}_i \emtx 
  +& \\
\bmtx  \tilde{C}^T\tilde{D}_{\psi 1}^T \\ \tilde{D}_{\psi 2}^T \emtx  \left(\bmtx \lambda_1 & \tilde{0}^T \\ \tilde{0} & \frac{\lambda_2}{n} I_n \emtx \otimes  \bmtx 0 & 1\\ 1 & 0\emtx\right)&
  \bmtx \tilde{D}_{\psi 1}\tilde{C} & \tilde{D}_{\psi2} \emtx
  \le 0
\end{split}
\end{align}
then  all $k\ge 1$ and $\xi^0\in \R^{n_\xi}$, the following inequality holds
\begin{align} 
\mathbb{E}\left[(\xi^{k+1}-\xi^*)^T(\tilde{P}\otimes I_p)(\xi^{k+1}-\xi^*)\right] \le \rho^2\mathbb{E}\left[ (\xi^k-\xi^*)^T(\tilde{P}\otimes I_p)(\xi^k-\xi^*)\right].
\end{align}

Consequently, $\mathbb{E}\left[\| \xi^k-\xi^*\|^2\right] \le \rho^{2k}\left( \textup{cond}(\tilde{P})\|\xi^0-\xi^*\|^2\right)$ holds
for all $k\ge 1$ and $\xi^0\in \R^{n_\xi}$, where $\textup{cond}$ denotes the condition number of a given positive definite matrix.
\end{theorem}
\begin{proof}
A detailed proof is presented in Section \ref{sec:ProofMain1}. Here we briefly sketch the proof idea.
Denote $P=\tilde{P}\otimes I_p$, and
  define a Lyapunov function by
  $V(\xi^k)= (\xi^k-\xi^*)^TP(\xi^k-\xi^*)$.
Then one can use the LMI condition and the basic quadratic inequalities \eqref{eq:basicIq1} \eqref{eq:basicIq2} to show that $V$ satisfies
$\mathbb{E}V(\xi^{k+1})-\rho^2 \mathbb{E}V(\xi^k) \le 0$.
This immediately leads to the desired conclusion. We can see the LMI condition gives us an automated way to search quadratic Lyapunov functions.
\end{proof}

The initial condition $\|\xi^0-\xi^*\|^2$ is related to the so-called variance term since $\xi^*$ is typically determined by $x^*$ and $\nabla f_i(x^*)$. When $\rho^2$ is given, the testing condition \eqref{eq:lmiIQCS0} is linear in $\tilde{P}$, $\lambda_1$, and $\lambda_2$. Therefore, \eqref{eq:lmiIQCS0} is an LMI whose feasible set is convex and can be effectively searched using the state-of-the-art convex optimization techniques, e.g. interior point method.  Many optimization solvers are available such that coding this  LMI condition  is a straightforward task. 

One can automate the proposed LMI analysis of stochastic optimization methods by modifying the values of $\nu$ and $\gamma$ to reflect various assumptions on $g$ and $f_i$. For SAG, SAGA and Finito, we always assume $g\in \mathcal{S}(m, L)$ with $m>0$ and hence we should set $\nu=-m$ in our analysis. The value of $\gamma$ is chosen based on the assumptions on $f_i$ as follows.
\begin{align*}
  \gamma =  \left\{
    \begin{array}{ll}
      -m & \mbox{if } f_i\in \mathcal{F}(m,L)\\
      0 & \mbox{if } f_i\in \mathcal{F}(0, L)\\
     L & \mbox{if } f_i \mbox{ is } L\mbox{-smooth} 
    \end{array}
  \right.
\end{align*}

For SDCA, \eqref{eq:minPl2} is considered. We assume $\frac{1}{n}\sum_{i=1}^n f_i\in \mathcal{F}(0,L)$. By co-coercivity, we can set $\nu=0$. In addition, we have $\gamma=0$ if $f_i\in \mathcal{F}(0, L)$ and $\gamma=L$   if $f_i$ is only assumed to be $L$-smooth.

\subsection{Numerical Pre-Analysis of Stochastic Methods Using Semidefinite Programs}
Theorem~\ref{thm:lmiS0} provides a simple unified tool for linear rate analysis of stochastic optimization methods governed by the general jump system model \eqref{eq:spfdJump0}. In principle, one can implement LMI \eqref{eq:lmiIQCS0} once. Then given a stochastic method \eqref{eq:spfdJump0}, one only needs to modify the $(\tilde{A}_i, \tilde{B}_i, \tilde{C})$ matrices in the code.
 Notice the size of the LMI condition \eqref{eq:lmiIQCS0} scales proportionally with $n$, and hence we can only solve LMI \eqref{eq:lmiIQCS0} numerically for $n$ up to several hundred. However, these numerical results with $n$ being several hundred 
provide informative clues for further proof constructions.
Notice the following two questions are important when analyzing a finite-sum method using Lyapunov arguments:
\begin{enumerate}
\item Which inequalities describing the function properties should be used in the proof?
\item What is the simplest form of Lyapunov function required by the proof? 
\end{enumerate}

Answers to these questions in the early stage of the analysis can guide researchers in their search for proofs. Usually one has to make a rough guess based on personal expertise. Theorem \ref{thm:lmiS0} provides a complementary numerical tool for this purpose. The numerical feasibility results from LMI \eqref{eq:lmiIQCS0} with $n$ being several hundred roughly answer the questions above by providing clues for selecting related function inequalities and simplified forms of Lyapunov functions.  
 For example, numerical tests of LMI \eqref{eq:lmiIQCS0} for SAGA show that enforcing the Lyapunov function to be diagonal does not change the feasibility results. This suggests using a diagonal Lyapunov function for SAGA.  When analyzing Finito, the numerical tests of \eqref{eq:lmiIQCS0} immediately indicate that Finito requires Lyapunov functions with off-diagonal terms. 
When  we test the existing rate results for SAG \citep[Theorem 1]{Schmidt2013},  LMI \eqref{eq:lmiIQCS0} becomes infeasible. This indicates that the analysis of SAG requires less conservative function inequalities in addition to the simple quadratic inequalities \eqref{eq:basicIq1} \eqref{eq:basicIq2}.
 The details of the numerical tests of LMI \eqref{eq:lmiIQCS0} are presented in Appendix \ref{sec:numerical}. 
Notice our proposed analysis heavily relies on the quadratic constraints used in the LMI formulations. 
Some stochastic methods, e.g. SAGA, SDCA and Finito, are relatively easier to analyze, since they only require the simple quadratic inequalities \eqref{eq:basicIq1} \eqref{eq:basicIq2}. Some other methods, e.g. SAG, are more involved, and require more advanced quadratic constraints in addition to  \eqref{eq:basicIq1} \eqref{eq:basicIq2}. Theorem \ref{thm:lmiS0} provides a simple tool to distinguish these two classes of stochastic methods.
We will further discuss SAG in Section \ref{sec:FDSAG}. Next,  we reduce LMI \eqref{eq:lmiIQCS0} to some equivalent small LMIs for SAGA, Finito, and SDCA.

\subsection{Dimension Reduction for the Proposed LMI}

The preliminary numerical test results of LMI \eqref{eq:lmiIQCS0} actually shed light on possible simplifications of the proposed LMI condition. 
Based on the preliminary numerical tests documented in Appendix \ref{sec:numerical},  it seems that \eqref{eq:lmiIQCS0} is sufficient for analysis of SAGA, Finito, and SDCA. 
As mentioned before, we notice various simplified parameterizations of $\tilde{P}$ are required for different algorithms. These simplified parameterizations seem not to introduce further conservatism into our analysis. The resultant LMI \eqref{eq:lmiIQCS0} with such $\tilde{P}$ consists of blocks which have the special form  $\mu I_n + q ee^T$ where $\mu$ and $q$ are some scalars. 
We summarize our preliminary findings in Table~\ref{tab:NaiveA}.
\begin{table}[h]
\resizebox{1\textwidth}{!}{\begin{minipage}{\textwidth}
\centering
\begin{tabular}{c|c|c}
\hline
 Method & Parameterization of $\tilde{P}$ &  Matrix Form of the Resultant LMI \eqref{eq:lmiIQCS0} \\ \hline\rule{0pt}{8mm}
SAGA  & $\bmtx p_1 I_n & \tilde{0} \\  \tilde{0}^T &  p_2 \emtx$ & $\bmtx \mu_{1} I_n +q_1 ee^T & q_4 e & \mu_6 I_n+q_6 ee^T\\ q_4 e^T & \mu_2  &  q_5 e^T \\ \mu_6 I_n+q_6 ee^T & q_5 e & \mu_3 I_n+ q_3 ee^T\emtx$  \\[5mm]  
SDCA  & $p_1 I_n+p_2 ee^T$ & $\bmtx \mu_{1} I_n +q_1 ee^T & \mu_3 I_n +q_3 ee^T \\ \mu_3 I_n +q_3 ee^T & \mu_2 I_n+q_2 ee^T\emtx$  \\[4mm] 
Finito  & $\bmtx p_1 I_n+p_2 ee^T & p_3 ee^T \\ p_3 ee^T & p_4 I_n+ p_5 ee^T\emtx$ & $\bmtx \mu_{1} I_n +q_1 ee^T & \mu_4 I_n +q_4 ee^T & \mu_6 I_n+q_6 ee^T\\ \mu_4 I_n +q_4 ee^T & \mu_2 I_n+q_2 ee^T & \mu_5 I_n+q_5 ee^T \\ \mu_6 I_n+q_6 ee^T & \mu_5 I_n +q_5 ee^T & \mu_3 I_n+ q_3 ee^T\emtx$  \\[5mm]  \hline
\end{tabular}
\caption{Parameterization of $\tilde{P}$ and Matrix Forms in \eqref{eq:lmiIQCS0} for SAGA, SDCA and Finito}
\label{tab:NaiveA}
\end{minipage} }
\end{table}

The special matrix  form of \eqref{eq:lmiIQCS0} is due to the same assumption on $f_i$ for all $i$ and the uniform sampling of $i_k$.
We can take advantage of the special matrix forms and convert \eqref{eq:lmiIQCS0} into equivalent small LMIs whose sizes do not depend on $n$. For example, we know 
$\bsmtx \mu_{1} I_n +q_1 ee^T & \mu_3 I_n +q_3 ee^T \\ \mu_3 I_n +q_3 ee^T & \mu_2 I_n+q_2 ee^T\esmtx\le 0$
if and only if
$\bsmtx \mu_1 & \mu_3 \\ \mu_3 & \mu_2\esmtx \le 0$
and $\bsmtx \mu_1 & \mu_3 \\ \mu_3 & \mu_2\esmtx + n \bsmtx q_1 & q_3 \\ q_3 & q_2 \esmtx \le 0$. Hence the analysis of SDCA actually involves two coupled $2\times 2$ LMIs. Similar linear algebra tricks can be used to convert \eqref{eq:lmiIQCS0} into equivalent small LMIs for SAGA and Finito. This leads to the following simplified testing conditions.

\begin{theorem}
\label{thm:LMISector}
 Suppose $i_k$ is uniformly sampled and $m>0$. Let a testing rate $0\le \rho\le 1$ be given.
\begin{enumerate}
\item (SAGA): Suppose $g\in \mathcal{S}(m, L)$, and $\gamma$ is defined by \eqref{eq:gammav} based on assumptions on $f_i$. If there exist positive scalars $p_1$, $p_2$, and non-negative scalars $\lambda_1$, $\lambda_2$ such that
\begin{align}
\label{eq:SAGALMI1}
\bmtx p_2 \alpha^2+\left(\frac{n-1}{n}-\rho^2\right)np_1 & -\alpha^2 p_2\\ -\alpha^2 p_2 & p_1+\alpha^2 p_2-2\lambda_2 \emtx \le 0\\
\label{eq:SAGALMI2}
\bmtx   (1-\rho^2)p_2-2\lambda_1 mL +2\lambda_2L\gamma & -\alpha p_2+(m+L)\lambda_1+(L-\gamma)\lambda_2\\ -\alpha p_2+(m+L)\lambda_1+(L-\gamma)\lambda_2 & p_1+\alpha^2 p_2-2\lambda_2-2\lambda_1 \emtx \le 0
\end{align}
Then SAGA  \eqref{eq:SAGAx} \eqref{eq:SAGy} initialized with any $x^0\in \R^p$ and $y_i^0\in \R^p$ satisfies
\begin{align}
\mathbb{E}\left[\| x^k-x^*\|^2+\frac{p_1}{p_2}\sum_{i=1}^n\|y_i^k-\nabla f_i(x^*)\|^2\right] \le \rho^{2k} R^0
\end{align}
where $R^0=\|x^0-x^*\|^2+\frac{p_1}{p_2}\sum_{i=1}^n\|y_i^0-\nabla f_i(x^*)\|^2$.
\item (Finito):  Suppose $g\in \mathcal{S}(m, L)$, and $\gamma$ is defined by \eqref{eq:gammav} based on assumptions on $f_i$. If there exist scalars $p_1$, $p_2$, $p_3$, $p_4$, $p_5$ and non-negative scalars $\lambda_1$, $\lambda_2$ such that $p_1>0$, $p_4>0$ and 
\begin{align}
&\bmtx p_1+np_2 & np_3 \\ np_3 & p_4+np_4\emtx>0\\
\label{eq:FinitoLMI10}
&\bmtx  p_2-p_1+n(1-\rho^2)p_1 & p_3 & -p_2\\ p_3 & p_5-p_4+n(1-\rho^2)p_4 & -p_3\\ -p_2 & -p_3 & p_1+p_2-2\lambda_2\emtx \le 0\\
\label{eq:FinitoLMI20}
&\bmtx X_{11} & X_{12} & X_{13}\\ X_{12} & (p_4+np_5)(1-\rho^2)-\frac{2Lm\lambda_1-2L\gamma \lambda_2}{n} &p_3+\frac{(L+m)\lambda_1+(L-\gamma)\lambda_2}{n} \\ X_{13} & p_3+\frac{(L+m)\lambda_1+(L-\gamma)\lambda_2}{n} & \frac{p_1+p_2-2\lambda_1-2\lambda_2}{n} \emtx\le 0\\
\begin{split}
&X_{11}= (1-\frac{1}{n}-\rho^2)p_1+\frac{p_2}{n}-n\rho^2 p_2+(n-2)p_2-2(1-\frac{1}{n})p_3\alpha n\\
&\,\,\,\,\,\,\,\,\,\,\,\,\,\,\,\,\,\,\,\,\,+(p_4+p_5-2Lm\lambda_1+2L\gamma \lambda_2)\alpha^2 n\\
\end{split}\\
&X_{12}= (1-\rho^2)p_3 n-p_3-(p_4+np_5-2Lm\lambda_1+2L\gamma\lambda_2)\alpha \\
&X_{13}=(1-\frac{1}{n})p_2-(p_3+\lambda_1(L+m)+\lambda_2(L-\gamma))\alpha
\end{align}
Then Finito \eqref{eq:Finitox} \eqref{eq:Finitoy} with any initial condition $x_i^0\in \R^p$ and $y_i^0\in \R^p$ satisfies
\begin{align}
\mathbb{E} V^k \le \rho^{2k}V^0 
\end{align}
where $V^k=(\xi^k-\xi^*)^T P (\xi^k-\xi^*)$, $\xi^k=\bmtx y^k \\ x^k \emtx$, $P=\bmtx p_1 I_n+p_2 ee^T & p_3 ee^T \\ p_3 ee^T & p_4 I_n+p_5 ee^T\emtx\otimes I_p$. 
\item (SDCA): Suppose $\frac{1}{n}\sum_{i=1}^n f_i \in \mathcal{F}(0, L)$. Set $\gamma=0$ if $f_i\in \mathcal{F}(0,L)$, and set $\gamma=L$ if $f_i$ is only $L$-smooth. Denote $\tilde{\alpha}=\alpha mn$. If there exist real scalars  $p_1$, $p_2$ and nonnegative $\lambda_1$, $\lambda_2$ such that $p_1>0$, $p_1+np_2>0$, and 
\begin{align}
\label{eq:SDCALMI1}
&\bmtx  p_1(\tilde{\alpha}^2-2\tilde{\alpha}+n(1-\rho^2))+p_2\tilde{\alpha}^2& p_1(\tilde{\alpha}^2-\tilde{\alpha})+\tilde{\alpha}^2 p_2\\  p_1(\tilde{\alpha}^2-\tilde{\alpha})+\tilde{\alpha}^2 p_2\ &  (p_1+p_2)\tilde{\alpha}^2-2\lambda_2\emtx \le 0\\
\label{eq:SDCALMI2}
&\bmtx X_{11} & X_{12} \\[2mm]  X_{12} &  (p_1+p_2)\tilde{\alpha}^2-2(\lambda_1+\lambda_2)\emtx \le 0\\
&X_{11}=p_1(\tilde{\alpha}^2-2\tilde{\alpha}+n(1-\rho^2))+p_2(\tilde{\alpha}-n)^2-n^2\rho^2p_2 +\frac{2\gamma L \lambda_2}{m^2}\\
&X_{12}=p_1(\tilde{\alpha}^2-\tilde{\alpha})+\tilde{\alpha}(\tilde{\alpha}-n) p_2+\frac{\lambda_1 L +(L-\gamma)\lambda_2}{m}
\end{align}
Then SDCA \eqref{eq:SDCAx} \eqref{eq:SDCAy} with stepsize $\alpha$ and initial condition $y_0^k$ satisfies
\begin{align}
\mathbb{E}\left[\| x^k-x^*\|^2+\frac{p_1}{p_2 m^2n^2}\sum_{i=1}^n\|y_i^k+\nabla f_i(x^*)\|^2\right] \le \rho^{2k}R^0
\end{align}
where $R^0=\|x^0-x^*\|^2+\frac{p_1}{p_2m^2 n^2}\sum_{i=1}^n\|y_i^0+\nabla f_i(x^*)\|^2$.
\end{enumerate}
\end{theorem}
\begin{proof}
One can compute analytical expressions of the matrix on the left side of \eqref{eq:lmiIQCS0} and prove this theorem using the linear algebra tricks mentioned before. Detailed proofs are left to Appendix \ref{sec:ProofDR}.
\end{proof}

\subsection{New Analytical Rate Bounds for SAGA, Finito, and SDCA}
We can analytically solve the LMIs in Theorem~\ref{thm:LMISector}, and prove the following rate results for SAGA, Finito and SDCA.
\begin{corollary}(Rate Bounds for SAGA)
\label{cor:SAGA}
Assume $i_k$ is uniformly sampled from $\mathcal{N}$, and $g\in \mathcal{S}(m,L)$ with $m>0$.
Consider SAGA \eqref{eq:SAGAx} \eqref{eq:SAGy} initialized from $x^0\in \R^p$ and $y_i^0\in \R^p$.  
\begin{enumerate}
\item If $f_i\in \mathcal{F}(m,L)$, then for any $0<\alpha\le \frac{1}{2L}$, one has
\begin{align}
\label{eq:SAGArateSC1}
\mathbb{E}\left[\| x^k-x^*\|^2\right] \le \left(1-\min{\left\{\frac{2L\alpha-1}{(L\alpha-1)n}, 2m\alpha-\frac{\alpha m^2}{(1-L\alpha)L}\right\}}\right)^k R^0
\end{align}
where $R^0=\|x^0-x^*\|^2+\frac{\alpha}{L}\sum_{i=1}^n\|y_i^0-\nabla f_i(x^*)\|^2$.
The following bound also holds for any $\alpha\le \frac{4}{9L}$
\begin{align}
\label{eq:SAGArateSC2}
\mathbb{E}\left[\| x^k-x^*\|^2\right] \le 
\left(1-\min{\left\{\frac{9L\alpha-4}{(3L\alpha-4)n}, 2m\alpha-\frac{3\alpha m^2}{(4-3L\alpha)L}\right\}}\right)^k R^0
\end{align}
where $R^0=\|x^0-x^*\|^2+\frac{2\alpha}{3L}\sum_{i=1}^n\|y_i^0-\nabla f_i(x^*)\|^2$.
\item If $f_i \in \mathcal{F}(0,L)$, then for any $0< \alpha \le \frac{1}{2L}$, one has
\begin{align}
\label{eq:SAGArateSC3}
\mathbb{E}\left[\| x^k-x^*\|^2\right] \le 
\left(1-\min{\left\{\frac{2L\alpha-b}{(L\alpha-b)n}, 2(1-b)m\alpha-\frac{\alpha m^2(1-b)^2}{(2-b-L\alpha)L}\right\}}\right)^k R^0
\end{align}
where $b$ can be any scalar in $[2L\alpha, 1]$, and $R^0=\|x^0-x^*\|^2+\frac{b\alpha}{L}\sum_{i=1}^n\|y_i^0-\nabla f_i(x^*)\|^2$. More specifically, when $\alpha=\frac{1}{3L}$, we can set $b=\frac{5}{6}$ and get the following bound:
\begin{align}
\label{eq:SAGArateSC4}
\mathbb{E}\left[\| x^k-x^*\|^2\right] \le 
\left(1-\min{\left\{\frac{1}{3n}, \frac{m}{10L}\right\}}\right)^k R^0
\end{align}
where $R^0=\|x^0-x^*\|^2+\frac{5}{18L^2}\sum_{i=1}^n\|y_i^0-\nabla f_i(x^*)\|^2$.
\item If $f_i$  is only assumed to be $L$-smooth, then the following bound holds for any $\alpha \le \frac{3m}{8L^2}$, 
\begin{align}
\label{eq:SAGArateSC5}
\mathbb{E}\left[\| x^k-x^*\|^2\right] \le 
\left(1-\min{\left\{\frac{b-2}{(b-1)n},  \frac{3m\alpha}{2}-2bL^2\alpha^2\right\}}\right)^k R^0
\end{align}
where $b$ can be any  scalar satisfying $2\le b \le \frac{3m}{4\alpha L^2}$, and $R^0=\|x^0-x^*\|^2+b\alpha^2\sum_{i=1}^n\|y_i^0-\nabla f_i(x^*)\|^2$. Specifically, when $\alpha=\frac{m}{8L^2}$, we can set $b=3$ and get the following bound:
\begin{align}
\label{eq:SAGArateSC6}
\mathbb{E}\left[\| x^k-x^*\|^2\right] \le 
\left(1-\min{\left\{\frac{1}{2n}, \frac{3m^2}{32L^2}\right\}}\right)^k R^0
\end{align}
where $R^0=\|x^0-x^*\|^2+\frac{3m^2}{64L^4}\sum_{i=1}^n\|y_i^0-\nabla f_i(x^*)\|^2$.
When $\alpha=\frac{m}{4(m^2n + L^2)}$, we can set $b=\frac{2(m^2n +L^2)}{L^2}$ and obtain
\begin{align}
\label{eq:SAGArateSC7}
\mathbb{E}\left[\| x^k-x^*\|^2\right] \le 
\left(1-\frac{m^2}{8(m^2n+L^2)}\right)^k R^0
\end{align}
where $R^0=\|x^0-x^*\|^2+\frac{m^2}{8(m^2n+L^2)L^2}\sum_{i=1}^n\|y_i^0-\nabla f_i(x^*)\|^2$. Hence, the $\epsilon$-optimal iteration complexity of SAGA without individual convexity is $\tilde{\mathcal{O}}\left((\frac{L^2}{m^2}+n)\log(\frac{1}{\epsilon})\right)$.
\end{enumerate}
\end{corollary}

\begin{corollary}(Rate Bounds for Finito)
\label{cor:Finito}
Assume $i_k$ is uniformly sampled from $\mathcal{N}$, and $g\in \mathcal{S}(m,L)$ with $m>0$.
Consider Finito \eqref{eq:Finitox} \eqref{eq:Finitoy}  initialized from $x_i^0\in \R^p$ and $y_i^0\in \R^p$.  Define  $v^k= \frac{1}{n}\sum_{i=1}^n x_i^k-\alpha\sum_{i=1}^n  y_i^k$.
\begin{enumerate}
\item If $f_i\in \mathcal{F}(m,L)$ and $n\ge \sqrt{\frac{50L}{m}}$, then Finito with $\alpha=\frac{1}{5L}$ satisfies
\begin{align}
\label{eq:FinitorateSC1}
\mathbb{E}\left[\frac{m}{10L}\sum_{i=1}^n\| x_i^k-x^*\|^2+\|v^k-x^*\|^2\right] \le \left(1-\min{\left\{\frac{1}{2n}, \frac{m}{20L}\right\}}\right)^k R^0
\end{align}
where $R^0=\frac{m}{10L}\sum_{i=1}^n\| x_i^0-x^*\|^2+\frac{1}{5L^2}\sum_{i=1}^n\|y_i^0-\nabla f_i(x^*)\|^2+\|v^0-x^*\|^2$.
\item If $f_i \in \mathcal{F}(0,L)$ and $n\ge \sqrt{\frac{64L}{m}}$, then Finito with $\alpha=\frac{1}{8L}$ satisfies 
\begin{align}
\label{eq:FinitorateSC2}
\mathbb{E}\left[\frac{m}{16L}\sum_{i=1}^n\| x_i^k-x^*\|^2+\|v^k-x^*\|^2\right] \le \left(1-\min{\left\{\frac{1}{3n}, \frac{5m}{176L}\right\}}\right)^k R^0
\end{align}
where $R^0=\frac{m}{16L}\sum_{i=1}^n\| x_i^0-x^*\|^2+\frac{1}{16L^2}\sum_{i=1}^n\|y_i^0-\nabla f_i(x^*)\|^2+\|v^0-x^*\|^2$.

\item If $f_i$  is $L$-smooth and $n\ge \frac{48L^2}{m^2}$, then Finito with $\alpha=\frac{1}{2nm}$ satisfies
\begin{align}
\label{eq:FinitorateSC3}
\mathbb{E}\left[\frac{3}{8n}\sum_{i=1}^n\| x_i^k-x^*\|^2+\|v^k-x^*\|^2\right] \le \left(1-\frac{1}{3n}\right)^k R^0
\end{align}
where $R^0=\frac{3}{8n}\sum_{i=1}^n\| x_i^0-x^*\|^2+\frac{1}{n^2m^2}\sum_{i=1}^n\|y_i^0-\nabla f_i(x^*)\|^2+\|v^0-x^*\|^2$.
\end{enumerate}
\end{corollary}

\begin{corollary}(Rate Bounds for SDCA without Duality)
\label{cor:SDCA}
Assume $i_k$ is uniformly sampled from $\mathcal{N}$, and $\sum_{i=1}^n f_i \in \mathcal{F}(0,L)$.
Consider SDCA \eqref{eq:SDCAx} \eqref{eq:SDCAy} initialized from $y_i^0$.  
\begin{enumerate}
\item If $f_i\in \mathcal{F}(0,L)$, then for any $0<\alpha\le \frac{2}{L+2mn}$, one has
\begin{align}
\label{eq:SDCArateSC1}
\mathbb{E}\left[\| x^k-x^*\|^2+\frac{\alpha }{(1-\alpha m n)mn}\sum_{i=1}^n \|y_i^k+\nabla f_i(x^*)\|^2\right] \le \left(1-m\alpha \right)^k R^0
\end{align}
where $R^0=\|x^0-x^*\|^2+\frac{\alpha}{(1-\alpha m n) m n}\sum_{i=1}^n\|y_i^0+\nabla f_i(x^*)\|^2$. 

\item If $f_i$ is $L$-smooth, then  \eqref{eq:SDCArateSC1} holds for any $0< \alpha \le \frac{m}{L^2+m^2 n}$.
When $\alpha=\frac{m}{(m^2n + L^2)}$, the following bound holds
\begin{align}
\label{eq:SDCArateSC3}
\mathbb{E}\left[\| x^k-x^*\|^2+\frac{1 }{L^2 n}\sum_{i=1}^n \|y_i^k+\nabla f_i(x^*)\|^2\right] \le  
\left(1-\frac{m^2}{m^2n+L^2}\right)^k R^0
\end{align}
where $R^0=\|x^0-x^*\|^2+\frac{1}{L^2 n}\sum_{i=1}^n\|y_i^0+\nabla f_i(x^*)\|^2$. 
\end{enumerate}
\end{corollary}

All the proofs are presented in Section \ref{sec:MainProof}. All three corollaries are actually proved via analytically solving the LMI conditions in Theorem \ref{thm:LMISector}. When $f_i$ is assumed to be only smooth (not necessarily convex), we only need to modify the value of $\gamma$ to be $L$ and then analytically construct a feasible solution for the resultant LMIs. We believe our rate bounds for SAGA and Finito without individual convexity (Statement 3 in Corollary \ref{cor:SAGA} and Statement 3 in Corollary \ref{cor:Finito}) are new. 
Now we briefly discuss the connections between our results and some existing rate bounds.

\begin{enumerate}
\item (SAGA) Statement 1  in Corollary \ref{cor:SAGA} is new in the sense that it works for a range of $\alpha$ and  also highlights the trade-off between the dependence of $\rho^2$ on $n$ and $\frac{m}{L}$. Notice that  \eqref{eq:SAGArateSC1} works better under the big data condition  while \eqref{eq:SAGArateSC2} is less conservative with large condition number $L/m$.  Suppose $f_i\in \mathcal{F}(m,L)$. If one chooses $\alpha=\frac{1}{3L}$ in \eqref{eq:SAGArateSC2} and applies the fact $L\ge m$, \eqref{eq:SAGArateSC2} directly leads to
\begin{align}
\label{eq:SAGArateSC8}
\mathbb{E}\left[\| x^k-x^*\|^2\right] \le 
\left(1-\min{\left\{\frac{1}{3n}, \frac{m}{3L}\right\}}\right)^k \left(\|x^0-x^*\|^2+\frac{2}{9L^2}\sum_{i=1}^n\|y_i^0-\nabla f_i(x^*)\|^2\right)
\end{align}
The convergence rate in the above bound agrees with the result in \citet[Section 2]{defazio2014}. On the other hand, one can also choose $\alpha=\frac{1}{3L}$ in \eqref{eq:SAGArateSC1} and obtain
\begin{align}
\label{eq:SAGArateSC9}
\mathbb{E}\left[\| x^k-x^*\|^2\right] \le 
\left(1-\min{\left\{\frac{1}{2n}, \frac{m}{6L}\right\}}\right)^k \left(\|x^0-x^*\|^2+\frac{1}{3L^2}\sum_{i=1}^n\|y_i^0-\nabla f_i(x^*)\|^2\right)
\end{align}
Clearly, the above bound is better than \eqref{eq:SAGArateSC8} under the big data condition $n\ge \frac{3L}{m}$. 
In principle, one can generate a family of bounds to describe this trade-off in more details. But all these bounds will only affect the iteration complexity $\tilde{\mathcal{O}}\left((n+\frac{L}{m})\log(\frac{1}{\epsilon})\right)$ by a constant factor.

Actually, we can also recover some other existing rate bounds for SAGA with individual convexity by modifying the proofs.
See Remark \ref{rem:SAGA} for further discussions.

We notice that
for any fixed $m$ and $L$, SAGA (with $n$ sufficiently large) can achieve a rate $\rho^2=1-\frac{1}{cn}$ where $c$ is arbitrarily close to $1$.  
For example, consider $f_i\in \mathcal{F}(m,L)$.
 Given any  $c \in (1, \infty)$, we can choose a sufficiently small $\alpha$
to ensure $\frac{L\alpha-1}{2L\alpha-1}<c$. For this specific value of $\alpha$,
 \eqref{eq:SAGArateSC1} just leads to a rate bound $\rho^2=1-\frac{1}{cn}$ under the condition $\left( 2m\alpha-\frac{\alpha m^2}{(1-L\alpha)L}\right)n\ge  \frac{2L\alpha-1}{L\alpha-1}$. Similar arguments also work when $f_i\in \mathcal{F}(0,L)$ or $f_i$ being $L$-smooth. 

\item (Finito): 
When $f_i\in \mathcal{F}(m,L)$ with $m>0$, our result states a linear rate bound for $\alpha=\frac{1}{5L}$, which is a stepsize independent of the parameter $m$. This could be useful since sometimes $m$ is unknown for practical problems.  On the other hand, the rate proofs in \citet[Theorem 1]{defazio2014finito} work for $\alpha=\frac{1}{2nm}$ under the big data condition $n\ge \frac{2L}{m}$. 

In general, our rate bounds for Finito are not as good as the rate bounds for SAGA. This is due to the fact that the LMI conditions for Finito are more complicated and involve more decision variables. We are only able to analytically solve these LMIs under the big data condition, although our preliminary numerical tests on the feasibility of  these LMIs suggest that Finito and SAGA have similar convergence rates.

\item (SDCA) Statement 1 in the above corollary is very similar to \citet[Theorem 1]{shalev2015sdca}. Actually, when $\alpha=\frac{1}{L+mn}$, \eqref{eq:SDCArateSC1} becomes
\begin{align}
\label{eq:SDCArateSC4}
\mathbb{E}\left[\| x^k-x^*\|^2+\frac{1 }{Lmn}\sum_{i=1}^n \|y_i^k+\nabla f_i(x^*)\|^2\right] \le \left(1-\frac{m}{L+mn} \right)^k R^0
\end{align}
where $R^0=\| x^0-x^*\|^2+\frac{1 }{Lmn}\sum_{i=1}^n \|y_i^0+\nabla f_i(x^*)\|^2$. This is almost identical to \citet[Theorem 1]{shalev2015sdca}. Statement 1 in Corollary \ref{cor:SDCA} is slightly stronger since it only requires $\alpha\le \frac{2}{L+2mn}$. Notice \citet[Theorem 1]{shalev2015sdca} requires $\alpha\le \frac{1}{L+mn}$. Similarly, Statement 2 in Corollary \ref{cor:SDCA} slightly improves \citet[Theorem 2]{shalev2015sdca} by allowing a slightly larger value of $\alpha$.
\end{enumerate}

\subsection{Further Discussion on SAG}
\label{sec:FDSAG}
Finally, we explain why Theorem \ref{thm:lmiS0} fails in recovering the existing SAG rate bounds in 
\citet[Theorem 1]{Schmidt2013}, and briefly sketch how to extend our LMI-based analysis for SAG.
The fundamental reason is that the proof of \citet[Theorem 1]{Schmidt2013} requires $g\in \mathcal{F}(m,L)$, which is stronger than the condition $g\in \mathcal{S}(m,L)$.  Notice in Theorem \ref{thm:lmiS0}, we only incorporate one property of $g$, i.e.
\begin{align}
\label{eq:gradv}
\bmtx v^k-x^* \\ \nabla g(v^k)\emtx^T \bmtx -2mL I_p & (L+m) I_p \\  (L+m) I_p & -2 I_p\emtx  
\bmtx v^k-x^* \\ \nabla g(v^k)\emtx\ge 0
\end{align}
The above inequality couples $v^k$ with $x^*$, and is satisfied for any $g\in \mathcal{S}(m,L)$.
However, the proof for \citet[Theorem 1]{Schmidt2013}  actually relies on some advanced inequalities  \footnote{See (11) in \citet{Schmidt2013} for such an inequality.} coupling $f(v^{k+1})$ with $f(v^k)$. Such advanced inequalities typically require $g\in \mathcal{F}(m, L)$. In other words, the convexity of $g$ is required in the convergence proof of SAG while our proofs for SAGA and Finito hold for some non-convex $g$.

Here is a similar example. The linear convergence of the full gradient descent method does not require convexity of the objective function, and can be proved using a basic quadratic inequality similar to \eqref{eq:gradv}. However, the linear convergence of Nesterov's accelerated method cannot be proved using this simple inequality and relies on some advanced inequalities coupling the current iterates with the past iterates. These advanced inequalities decode convexity much better than the simple inequality used in the proof of the full gradient descent method.
One such advanced inequality is the so-called weighted off-by-one IQC \citep[Lemma 10]{Lessard2014}. 
See \citet[Section 4.5]{Lessard2014} for a detailed discussion on how to incorporate the weighted off-by-one IQC for analysis of Nesterov's accelerated method. The use of the weighted off-by-one IQC typically leads to larger LMIs which are difficult to solve analytically. Very recently, \citet{BinHu2017} have proposed another inequality of similar nature to simplify the LMI-based analysis of Nesterov's accelerated method. The resultant LMI in \citet{BinHu2017} is smaller and can be solved analytically to recover the standard rate of Nesterov's method. To summarize, more advanced quadratic inequalities which further exploit the property of convexity
are required in the analysis of Nesterov's accelerated method, and this makes the analysis of Nesterov's accelerated method much more complicated than the analysis of the full gradient descent method.

Due to similar reasons, the analysis of SAG is more involved than other stochastic methods. Our quadratic constraint approach actually reveals the difficulties in analyzing different methods: SAGA, SDCA, and Finito only require simple constraints \eqref{eq:basicIq1} \eqref{eq:basicIq2} while SAG further requires more advanced quadratic constraints, e.g. weighted off-by-one IQC.

Now we briefly sketch two ways to address the analysis of SAG. First, one can combine our proposed jump system theory with the quadratic constraint derivation procedure in \citet{BinHu2017}. We can obtain a modified LMI condition which searches for a Lyapunov function in the form of $\left((\xi^k-\xi^*)^T P(\xi^k-\xi^*)+g(v^k)-g(x^*)\right)$ where $P$ is some positive semidefinite matrix. We have some preliminary numerical rate results indicating that formulating such an LMI to search for Lyapunov functions in the more general form is sufficient to numerically analyze SAG. Actually, the original proof of  \citet[Theorem 1]{Schmidt2013} constructs such a Lyapunov function \cite[Section B.2]{Schmidt2013}. 

Another way to address the analysis of SAG is to incorporate the weighted off-by-one IQC \citep[Lemma 10]{Lessard2014} into our jump system framework. In this case, we can formulate an LMI condition to search for a quadratic function which is not a Lyapunov function in the technical sense but serves the purpose of linear convergence certifications. See \citet[Remarks on Lyapunov Functions]{Lessard2014} for more explanations.
We also have some preliminary numerical rate results suggesting that applying the weighted off-by-one IQC can recover the linear convergence rates in \citet[Theorem 1]{Schmidt2013} and lead to new linear rate bounds under various assumptions on $f_i$.

Although there is no technical difficulty in incorporating these more advanced quadratic constraints into the LMI formulations for SAG,
we have not been able to analytically solve these resultant LMIs. In addition, the use of such advanced quadratic constraints requires much heavier mathematical notation.  For readability purposes, we do not include a detailed numerical rate analysis of SAG in this paper. See \citet{Lessard2014} and \citet{BinHu2017} for detailed discussions on weighted off-by-one IQC and other more advanced quadratic constraints.

\section{Main Technical Proofs}
\label{sec:MainProof}
We present the proofs of Theorem \ref{thm:lmiS0}, Corollary \ref{cor:SAGA}, Corollary \ref{cor:Finito}, and Corollary \ref{cor:SDCA} in this section. The proof of Theorem \ref{thm:LMISector} is quite tedious, and hence left to Appendix \ref{sec:ProofDR}.
\subsection{Proof of the Main LMI Condition (Theorem \ref{thm:lmiS0})}
\label{sec:ProofMain1}

Based on the state space model in \eqref{eq:spfdJump0} and \eqref{eq:spfdJumpEqi}, we have
\begin{align}
\label{eq:Gshift1}
\begin{split}
\xi^{k+1}-\xi^*&=A_{i_k} (\xi^k-\xi^*)+B_{i_k}(w^k-w^*)\\
v^k-v^*&=C(\xi^k-\xi^*)
\end{split}
\end{align}

Denote $P=\tilde{P}\otimes I_p$, and
  define the Lyapunov function by
  $V(\xi^k)= (\xi^k-\xi^*)^TP(\xi^k-\xi^*)$.  Based on \eqref{eq:Gshift1}, we have the following key relation:
\begin{align}
\label{eq:keyrelation1}
\begin{split}
&\mathbb{E}[V(\xi^{k+1})\,\vert \,\mathcal{F}_{k-1}]\\
=&\mathbb{E}[ (\xi^{k+1}-\xi^*)^TP(\xi^{k+1}-\xi^*)\,\vert\, \mathcal{F}_{k-1}]\\
=& \sum_{i=1}^n \mathbb{P}(i_k=i)\left[A_i(\xi^k-\xi^*)+B_i(w^k-w^*) \right]^TP \left[A_i(\xi^k-\xi^*)+B_i(w^k-w^*) \right]\\
=&\bmtx \xi^k-\xi^* \\ w^k-w^*\emtx^T   \bmtx \frac{1}{n}\sum_{i=1}^n  A_i^T P A_i  & \frac{1}{n}\sum_{i=1}^n A_i^T P B_i \\[2mm] 
  \frac{1}{n}\sum_{i=1}^n B_i^T P A_i & \frac{1}{n}\sum_{i=1}^n B_i^T P B_i \emtx  \bmtx \xi^k-\xi^* \\ w^k-w^*\emtx
\end{split}
\end{align}

Suppose  $D_{\psi 1}=\tilde{D}_{\psi 1}\otimes I_p$ and $D_{\psi 2}=\tilde{D}_{\psi 2}\otimes I_p$. 
Notice we always have
\begin{align}
\bmtx 2L\nu I_p & (L-\nu) I_p \\ (L-\nu) I_p & -2 I_p \emtx=\bmtx LI_p & -I_p\\ \nu I_p & I_p\emtx^T  \bmtx 0_p & I_p \\ I_p & 0_p \emtx \bmtx LI_p & -I_p\\ \nu I_p & I_p\emtx
\end{align}

Moreover, we have $C(\xi^k-\xi^*)=v^k-x^*$.
Hence another key relation also holds as follows
\begin{align}
\label{eq:keyrelation2}
\begin{split}
 &\bmtx \xi^k-\xi^* \\ w^k-w^*\emtx^T \bmtx  C^T D_{\psi 1}^T \\ D_{\psi 2}^T \emtx  \left(\bmtx \lambda_1 & \tilde{0}^T \\ \tilde{0} & \frac{\lambda_2}{n} I_n \emtx \otimes  \bmtx 0_p & I_p\\ I_p & 0_p\emtx\right)
  \bmtx D_{\psi 1}C & D_{\psi2} \emtx  \bmtx \xi^k-\xi^* \\ w^k-w^*\emtx\\
=&\lambda_1 \bmtx v^k-x^* \\ \frac{\sum_{i=1}^n (\nabla f_i(v^k)-\nabla f_i(x^*))}{n}\emtx^T \bmtx 2L\nu I_p & (L-\nu) I_p \\ (L-\nu) I_p & -2 I_p \emtx  \bmtx v^k-x^* \\ \frac{\sum_{i=1}^n (\nabla f_i(v^k)-\nabla f_i(x^*))}{n}\emtx\\
+&  \frac{\lambda_2}{n}\sum_{i=1}^n \bmtx v^k-x^* \\ \nabla f_i(v^k)-\nabla f_i(x^*)\emtx^T \bmtx 2L\gamma I_p & (L-\gamma) I_p \\ (L-\gamma) I_p & -2 I_p \emtx \bmtx v^k-x^* \\ \nabla f_i(v^k)-\nabla f_i(x^*)\emtx\ge 0
\end{split}
\end{align}
The last step follows from \eqref{eq:basicIq1} and \eqref{eq:basicIq2}, which are some simple quadratic inequalities capturing the properties of $f_i$.
Now we can take the Kronecker product of the left side of \eqref{eq:lmiIQCS0} with $I_p$ and immediately get
\begin{align}
\begin{split}
 \bmtx \frac{1}{n}\sum_{i=1}^n  A_i^T P A_i-\rho^2 P  & \frac{1}{n}\sum_{i=1}^n A_i^T P B_i \\[2mm] 
  \frac{1}{n}\sum_{i=1}^n B_i^T P A_i & \frac{1}{n}\sum_{i=1}^n B_i^T P B_i \emtx +& \\
\bmtx  C^T D_{\psi 1}^T \\ D_{\psi 2}^T \emtx  \left(\bmtx \lambda_1 & \tilde{0}^T \\ \tilde{0} & \frac{\lambda_2}{n} I_n \emtx \otimes  \bmtx 0_p & I_p\\ I_p & 0_p\emtx\right)&
  \bmtx D_{\psi 1}C & D_{\psi2} \emtx
  \le 0
\end{split}
\end{align}
Therefore, left and
  right multiply the above inequality  by $[(\xi^k-\xi^*)^T, (w^k-w^*)^T]$ and $[(\xi^k-\xi^*)^T, (w^k-w^*)^T]^T$ and apply \eqref{eq:keyrelation1}, \eqref{eq:keyrelation2}  to show that $V$ satisfies:
\begin{align}
\label{eq:Sdi1}
\mathbb{E}[V(\xi^{k+1})\,\vert\, \mathcal{F}_{k-1}]-\rho^2 V(\xi^k)  \le 0
\end{align}
 
We can take full expectation to get
$\mathbb{E}V(\xi^{k+1})-\rho^2 \mathbb{E}V(\xi^k) \le 0$.
Consequently, we immediately have $\mathbb{E}V(\xi^k)\le \rho^{2k} V(\xi^0)$ and
$\mathbb{E}[\|\xi^k-\xi^*\|^2] \le \rho^{2k}\left(\textup{cond}(P)\|\xi^0-\xi^*\|^2\right)$.

\subsection{Analytical Proof for SAGA (Corollary \ref{cor:SAGA})}
To prove Statement 1,  we set $\gamma=-m$ to reflect the assumption $f_i\in \mathcal{F}(m,L)$. 
Hence LMI \eqref{eq:SAGALMI2} becomes
\begin{align}
\bmtx (1-\rho^2)p_2-2\lambda_1 mL -2\lambda_2 mL & -\alpha p_2 +(L+m)(\lambda_1+\lambda_2) \\ -\alpha p_2 +(L+m)(\lambda_1+\lambda_2) & p_1+\alpha^2 p_2-2(\lambda_1+\lambda_2)\emtx \le 0
\end{align}

By Shur complements, LMIs \eqref{eq:SAGALMI1} \eqref{eq:SAGALMI2} are equivalent to 
\begin{align}
\label{eq:SAGALMI3}
&p_1+\alpha^2p_2-2\lambda_2\le 0\\
\label{eq:SAGALMInS}
&\rho^2\ge 1-\frac{1}{n}-\left(\frac{\alpha^4 p_2^2}{p_1+\alpha^2 p_2-2\lambda_2}-\alpha^2 p_2\right)\frac{1}{np_1}\\
\label{eq:SAGALMIrho1}
&\rho^2\ge 1-2(\lambda_1+\lambda_2) mLp_2^{-1} -\frac{(-\alpha p_2+(L+m)(\lambda_1+\lambda_2))^2}{(p_1+\alpha^2 p_2-2(\lambda_1+\lambda_2))p_2}
\end{align}
We can see that \eqref{eq:SAGALMInS} describes how $\rho^2$ depends on $n$, while \eqref{eq:SAGALMIrho1} describes how $\rho^2$ depends on $m$ and $L$. We need the common feasible set for both \eqref{eq:SAGALMInS} and  \eqref{eq:SAGALMIrho1}.

More formally, given the testing rate $\rho^2=1-\min{\left\{\frac{2L\alpha-1}{(L\alpha-1)n}, 2m\alpha-\frac{\alpha m^2}{(1-L\alpha)L}\right\}}$, it is straightforward to verify  $0\le \rho^2\le 1$ when $\alpha\le \frac{1}{2L}$.
For this particular rate, the condition \eqref{eq:SAGALMI3}  \eqref{eq:SAGALMInS} \eqref{eq:SAGALMIrho1}  is feasible with $p_1=\frac{1}{L}$, $p_2=\frac{1}{\alpha}$, $\lambda_1=0$, and $\lambda_2=\frac{1}{L}$. By Theorem \ref{thm:LMISector}, \eqref{eq:SAGArateSC1} holds as desired. Similarly, given the testing rate $\rho^2=1-\min{\left\{\frac{9L\alpha-4}{(3L\alpha-4)n}, 2m\alpha-\frac{3\alpha m^2}{(4-3L\alpha)L}\right\}}$, we can choose  $p_1=\frac{2}{3L}$, $p_2=\frac{1}{\alpha}$, $\lambda_1=0$,  and $\lambda_2=\frac{1}{L}$ to prove the bound \eqref{eq:SAGArateSC2}.  Therefore, Statement 1 is true.

To prove Statement 2, we set $\gamma=0$ in \eqref{eq:SAGALMI2} to reflect the assumption $f_i \in \mathcal{F}(0,L)$. Again, by Schur complements, LMIs \eqref{eq:SAGALMI1} \eqref{eq:SAGALMI2} are equivalent to \eqref{eq:SAGALMI3}  \eqref{eq:SAGALMInS} and 
\begin{align}
\label{eq:SAGALMIrho2}
\rho^2\ge 1-2\lambda_1 mLp_2^{-1} -\frac{(-\alpha p_2+(L+m)\lambda_1+L\lambda_2)^2}{(p_1+\alpha^2 p_2-2\lambda_1-2\lambda_2)p_2}
\end{align}

Given the testing rate $\rho^2=1-\min{\left\{\frac{2L\alpha-b}{(L\alpha-b)n}, 2(1-b)m\alpha-\frac{m^2(1-b)^2\alpha}{(2-b-L\alpha)L}\right\}}$, it is straightforward to verify  $0\le \rho^2\le 1$ when $b\ge 2L\alpha$.
For this particular rate, the condition~\eqref{eq:SAGALMI3}~\eqref{eq:SAGALMInS}~\eqref{eq:SAGALMIrho2}  is feasible with $p_1=\frac{b}{L}>0$, $p_2=\frac{1}{\alpha}$, $\lambda_1=\frac{1-b}{L}\ge 0$, and $\lambda_2=\frac{b}{L}$. By Theorem \ref{thm:LMISector}, \eqref{eq:SAGArateSC3} holds as desired. When $\alpha=\frac{1}{3L}$, we can choose any $b\in [\frac{2}{3}, 1]$ and \eqref{eq:SAGArateSC3} holds. Hence we can easily obtain \eqref{eq:SAGArateSC4} by choosing $b=\frac{5}{6}$ and applying the fact $\frac{m}{L}\le 1$.

To prove Statement 3, we set $\gamma=L$ in \eqref{eq:SAGALMI2} to reflect the assumption $f_i$ being $L$-smooth. Again, by Schur complements, LMIs \eqref{eq:SAGALMI1} \eqref{eq:SAGALMI2} are equivalent to \eqref{eq:SAGALMI3}, \eqref{eq:SAGALMInS} and 
\begin{align}
\label{eq:SAGALMIrho3}
\rho^2\ge 1-2\lambda_1 mLp_2^{-1} +2\lambda_2 L^2 p_2^{-1}-\frac{(-\alpha p_2+(L+m)\lambda_1)^2}{(p_1+\alpha^2 p_2-2\lambda_1-2\lambda_2)p_2}
\end{align}

Given the testing rate $\rho^2=1-\min{\left\{\frac{b-2}{(b-1)n},  \frac{3m\alpha}{2}-2bL^2\alpha^2\right\}}$, it is straightforward to verify  $0\le \rho^2\le 1$ when $2\le b \le \frac{3m}{4\alpha L^2}$.
For this particular rate, the condition \eqref{eq:SAGALMI3} \eqref{eq:SAGALMInS} \eqref{eq:SAGALMIrho3}  is feasible with $p_1=b\alpha>0$, $p_2=\frac{1}{\alpha}$, $\lambda_1=\frac{1}{L}\ge 0$, and $\lambda_2=b\alpha$. Notice the facts $m \le L$ and $b\ge 2$ are required when checking the feasibility of the LMI condition. By Theorem \ref{thm:LMISector}, \eqref{eq:SAGArateSC5} holds as desired. When $\alpha=\frac{m}{8L^2}$, we can choose any $b\in [2, 3]$ and \eqref{eq:SAGArateSC5} holds. Hence we can easily obtain \eqref{eq:SAGArateSC6} by choosing $b=3$ and applying the fact $\frac{m}{L}\le 1$. Similarly, when $\alpha=\frac{m^2}{4(m^2n +L^2)}$,  we can choose any $2\le b \le \frac{3(m^2n+L^2)}{L^2}$ and \eqref{eq:SAGArateSC5} holds. Hence we can also obtain \eqref{eq:SAGArateSC7} by choosing $b=\frac{2(m^2n+L^2)}{L^2}$ and apply the fact $\frac{2m^2}{2m^2 n+L^2}\ge \frac{m^2}{8(m^2 n+L^2)}$. This completes the proof.

\begin{remark}
\label{rem:SAGA}
Based on the above proof, we can actually recover two other known results in \citet{defazio2014}. First, it is known that SAGA achieves the rate $\rho^2=1-\frac{m}{2(mn+L)}$ given the assumption $f_i\in\mathcal{F}(m,L)$ and the stepsize $\alpha=\frac{1}{2(mn+L)}$. To recover this result, we first consider the case where $L\ge 2m$. Clearly $\alpha L<\frac{1}{2}$. Then the formula \eqref{eq:SAGArateSC1} leads to a rate
$\rho^2=1- \frac{m}{mn+L}+\frac{m^2}{(L+2mn)L}$.
If $L\ge 2m$, then the above rate bound is always better than $\rho^2=1-\frac{m}{2(mn+L)}$. On the other hand,
if $L\le 2m$, we can use $p_2=\frac{1}{\alpha}$, $\lambda_1=0$, $\lambda_2=\frac{1}{L}$ and $p_1=0.75\lambda_2$ to prove the LMI condition is feasible with 
$\rho^2=1-\min{\left\{\frac{15mn-L}{n(15mn+9L)}, \frac{m}{mn+L}-\frac{2m^2}{(3L+5mn)L}\right\}}$.
Under the condition $L\le 2m$, this rate bound is always lower than $1-\frac{m}{2(mn+L)}$. Consequently, we successfully recover the existing rate bound $\rho^2=1-\frac{m}{2(mn+L)}$ for $\alpha=\frac{1}{2(mn+L)}$.
 Second, when $f_i$ is only assumed to convex and smooth, i.e. $f_i\in \mathcal{F}(0,L)$, we can also choose $\alpha=\frac{1}{3(mn+L)}$  in \eqref{eq:SAGArateSC3} and set $b=\frac{2}{3}$. This leads to
\begin{align}
\begin{split} 
\mathbb{E}\left[\| x^k-x^*\|^2\right] &\le \left(1-\min{\left\{\frac{2m}{L+2mn}, \frac{2m}{9(L+mn)}-\frac{m^2}{27L^2+36mnL}\right\}}\right)^k R^0\\
&=\left(1-\frac{2m}{9(L+mn)}+\frac{m^2}{27L^2+36mnL}\right)^k R^0\\
&\le \left(1-\frac{m}{6(mn+L)}\right)^k R^0
\end{split}
\end{align}
where  $R^0=\|x^0-x^*\|^2+\frac{2}{9(mn+L)L}\sum_{i=1}^n\|y_i^0-\nabla f_i(x^*)\|^2$. 
The rate bound here is also consistent with the known result in \citet{defazio2014}.
\end{remark}

\subsection{Analytical Proof for Finito (Corollary \ref{cor:Finito})}
First, we need the following linear algebra result to relax the LMI conditions \eqref{eq:FinitoLMI10} \eqref{eq:FinitoLMI20} to some simpler testing conditions.
\begin{lemma}
\label{lem:lineartrick1}
Suppose $Y_{11}$, $Y_{12}$, $Y_{22}$, $\alpha$, and $n$ are scalars. In addition, $Y_{11}\le 0$, $Y_{22}\le 0$, $\alpha>0$ and $n>0$.
The following two statements are true.
\begin{enumerate}
\item  If $Y_{12}\le 0$, then $\bmtx Y_{11} +\alpha n Y_{12}& Y_{12} \\ Y_{12} & Y_{22}+\frac{Y_{12}}{\alpha n}\emtx\le 0$.
\item If $Y_{12}\ge 0$, then $\bmtx Y_{11}-\alpha n Y_{12} & Y_{12} \\ Y_{12} & Y_{22}-\frac{Y_{12}}{\alpha n}\emtx\le 0$.
\end{enumerate}
\end{lemma}
\begin{proof}
Statement 1 can be proved using the fact $\bsmtx \alpha n & 1\\ 1 & \frac{1}{\alpha n}\esmtx \ge 0$. Statement 2 can be proved using the fact $\bsmtx \alpha n & -1\\ -1 & \frac{1}{\alpha n}\esmtx \ge 0$.
\end{proof}

Next, we relax the LMIs \eqref{eq:FinitoLMI10} \eqref{eq:FinitoLMI20} to some simpler (but more conservative) testing conditions. The relaxed conditions are sufficiently useful for analysis of Finito under some big data condition.
\begin{corollary}
\label{cor:Finito2}
Consider Finito \eqref{eq:Finitox} \eqref{eq:Finitoy} with $i_k$ sampled from a uniform distribution. Define $v^k= \frac{1}{n}\sum_{i=1}^n x_i^k-\alpha\sum_{i=1}^n  y_i^k$. Suppose $g\in \mathcal{S}(m,L)$ with $m>0$, and $\gamma$ is defined by \eqref{eq:gammav} based on assumptions on $f_i$. Given any testing rate $1-\frac{1}{n}\le \rho^2 \le 1$, if there exist positive scalars $p_1$, $p_4$, and nonnegative scalars $\lambda_1$, $\lambda_2$ such that
\begin{align}
\alpha^2-2\lambda_2 +p_1< 0\\
\label{eq:FinitoY11}
n(1-\rho^2)p_1-p_1+2\alpha^2-\frac{2\alpha^4}{\alpha^2-2\lambda_2+p_1}\le 0 \\
\label{eq:FinitoY22}
n(1-\rho^2)p_4-p_4+\frac{2}{n^2}-\frac{2\alpha^2}{n^2(\alpha^2-2\lambda_2+p_1)}\le 0\\
\label{eq:FinitoZ12}
p_4-\rho^2+2L\gamma \lambda_2-2Lm \lambda_1+1-\frac{((L+m)\lambda_1+(L-\gamma)\lambda_2-\alpha)^2}{\alpha^2-2\lambda_1-2\lambda_2+p_1}\le 0
\end{align}
then Finito  \eqref{eq:Finitox} \eqref{eq:Finitoy} with any initial condition $x_i^0\in \R^p$ and $y_i^0\in \R^p$ satisfies
\begin{align}
\mathbb{E}\left[p_4\sum_{i=1}^n\| x_i^k-x^*\|^2+p_1\sum_{i=1}^n\|y_i^k-\nabla f_i(x^*)\|^2+\|v^k-x^*\|^2\right] \le \rho^{2k}R^0
\end{align}
where $R^0=p_4\sum_{i=1}^n\| x_i^0-x^*\|^2+p_1\sum_{i=1}^n\|y_i^0-\nabla f_i(x^*)\|^2+\|v^0-x^*\|^2$.
\end{corollary}
\begin{proof}
Consider $p_2=\alpha^2$, $p_3=-\frac{\alpha}{n}$, and $p_5=\frac{1}{n^2}$. Clearly, we have
\begin{align}
\bmtx p_1+np_2 & np_3 \\ np_3 & p_4+np_5\emtx=\bmtx p_1 & 0 \\ 0& p_4\emtx +n \bmtx -\alpha \\ \frac{1}{n}\emtx \bmtx -\alpha & \frac{1}{n}\emtx >0
\end{align}

 Applying Schur complement with respect to the $(3,3)$-entry of \eqref{eq:FinitoLMI10}, we can immediately rewrite \eqref{eq:FinitoLMI10} as $p_1+p_2-2\lambda_2=\alpha^2-2\lambda_1+p_1\le 0$ and $\bsmtx Y_{11} +\alpha n Y_{12}& Y_{12} \\ Y_{12} & Y_{22}+\frac{Y_{12}}{\alpha n}\esmtx\le 0$, where $Y_{11}$ is equal to the left side of \eqref{eq:FinitoY11}, $Y_{22}$ is equal to the left side of \eqref{eq:FinitoY22}, and $Y_{12}=\frac{\alpha^3}{n(\alpha^2-2\lambda_2+p_1)}-\frac{\alpha}{n}$. Similarly, we can apply Schur complement with respect to the $(3,3)$-entry of \eqref{eq:FinitoLMI20} and rewrite \eqref{eq:FinitoLMI20} as $p_1+p_2-2\lambda_1-2\lambda_2\le 0$ and $\bsmtx Z_{11} -\alpha n Z_{12}& Z_{12} \\ Z_{12} & Z_{22}-\frac{Z_{12}}{\alpha n}\esmtx\le 0$, where $Z_{11}=p_1(1-\rho^2-\frac{1}{n})$, $Z_{22}=p_4(1-\rho^2-\frac{1}{n})$, and $Z_{12}$ is equal to the multiplication of $\alpha$ and the left side of \eqref{eq:FinitoZ12}. Based on the conditions in the corollary statement, we can directly apply Lemma \ref{lem:lineartrick1} to show that \eqref{eq:FinitoLMI10} and \eqref{eq:FinitoLMI20} hold. Finally, notice 
\begin{align}
\bmtx p_1I_n+p_2 ee^T & p_3 ee^T \\ p_3 ee^T & p_4I_n+p_5ee^T\emtx=\bmtx p_1  & 0 \\ 0 & p_4 \emtx\otimes I_n+\bmtx -\alpha e \\ \frac{1}{n}e \emtx \bmtx -\alpha e^T & \frac{1}{n}e^T \emtx
\end{align}
We can directly apply Statement 3 in Theorem \ref{thm:LMISector} to complete the proof of this corollary.
\end{proof}

Now we can choose $p_1$, $p_4$, $\lambda_1$ and $\lambda_2$ to prove Corollary \ref{cor:Finito}.
Notice \eqref{eq:FinitoY11}, \eqref{eq:FinitoY22}, and \eqref{eq:FinitoZ12} are equivalent to
\begin{align}
\label{eq:FinitoY11n}
\rho^2&\ge 1-\frac{1}{n}+\frac{2\alpha^2(p_1-2\lambda_2)}{np_1(\alpha^2-2\lambda_2+p_1)}\\
\label{eq:FinitoY22n}
\rho^2&\ge 1-\frac{1}{n}+\frac{2(p_1-2\lambda_2)}{n^3p_4(\alpha^2-2\lambda_2+p_1)}\\
\label{eq:FinitoZ12n}
\rho^2&\ge 1-2Lm \lambda_1+2L\gamma \lambda_2+p_4-\frac{((L+m)\lambda_1+(L-\gamma)\lambda_2-\alpha)^2}{\alpha^2-2\lambda_1-2\lambda_2+p_1}
\end{align}
\begin{enumerate}
\item To prove Statement 1,  we set $\gamma=-m$ to reflect the assumption $f_i\in \mathcal{F}(m,L)$. 
We choose $p_1=\frac{\alpha}{L}$, $p_4=0.5m\alpha$, $\lambda_1=0$, and $\lambda_2=\frac{\alpha}{L}$. Then \eqref{eq:FinitoY11n}, \eqref{eq:FinitoY22n}, and \eqref{eq:FinitoZ12n} become
\begin{align}
\label{eq:FinitoCon1}
\rho^2&\ge 1-\frac{1}{n}+\frac{2\alpha L}{n(1-\alpha L)}\\
\label{eq:FinitoCon2}
\rho^2&\ge 1-\frac{1}{n}+\frac{4}{n^3 m \alpha (1-L\alpha)}\\
\label{eq:FinitoCon3}
\rho^2&\ge 1-1.5m \alpha+\frac{m^2\alpha}{L(1-L\alpha)}
\end{align}
When $\alpha=\frac{1}{5L}$, the testing rate $\rho^2=1-\min{\left\{\frac{1}{2n}, \frac{m}{20L}\right\}}$ satisfies \eqref{eq:FinitoCon1} and \eqref{eq:FinitoCon3}. In addition, this testing rate also satisfies \eqref{eq:FinitoCon2} under the further assumption $n\ge \sqrt{\frac{50L}{m}}$.  Therefore, Statement 1 directly follows from Corollary \ref{cor:Finito2}.

\item
To prove Statement 2, we set $\gamma=0$ to reflect the assumption $f_i \in \mathcal{F}(0,L)$. 
We choose $p_1=\frac{\alpha}{2L}$, $p_4=0.5m\alpha$, $\lambda_1=\frac{\alpha}{2L}$, and $\lambda_2=\frac{\alpha}{2L}$. Then \eqref{eq:FinitoY11n}, \eqref{eq:FinitoY22n}, and \eqref{eq:FinitoZ12n} become
\begin{align}
\label{eq:FinitoCon1b}
\rho^2&\ge 1-\frac{1}{n}+\frac{4\alpha L}{n(1-2\alpha L)}\\
\label{eq:FinitoCon2b}
\rho^2&\ge 1-\frac{1}{n}+\frac{4}{n^3 m \alpha (1-2L\alpha)}\\
\label{eq:FinitoCon3b}
\rho^2&\ge 1-0.5m \alpha+\frac{m^2\alpha}{2L(3-2L\alpha)}
\end{align}
When $\alpha=\frac{1}{8L}$, the testing rate $\rho^2=1-\min{\left\{\frac{1}{3n}, \frac{5m}{176L}\right\}}$ satisfies \eqref{eq:FinitoCon1b} and \eqref{eq:FinitoCon3b}. In addition, this testing rate also satisfies \eqref{eq:FinitoCon2b} under the further assumption $n\ge \sqrt{\frac{64L}{m}}$.  Therefore, Statement 2 directly follows from Corollary \ref{cor:Finito2}.

\item
To prove Statement 3, we set $\gamma=L$ to reflect the assumption $f_i$ being $L$-smooth. 
We choose $p_1=4\alpha^2$, $p_4=0.75m\alpha$, $\lambda_1=\frac{\alpha}{L}$, and $\lambda_2=4\alpha^2$. Then \eqref{eq:FinitoY11n}, \eqref{eq:FinitoY22n}, and \eqref{eq:FinitoZ12n} become
\begin{align}
\label{eq:FinitoCon1c}
\rho^2&\ge 1-\frac{1}{3n}\\
\label{eq:FinitoCon2c}
\rho^2&\ge 1-\frac{1}{n}+\frac{32}{9n^3 m \alpha}\\
\label{eq:FinitoCon3c}
\rho^2&\ge 1-1.25m \alpha+8L^2\alpha^2+\frac{m^2\alpha}{L(2+3L\alpha)}
\end{align}
When $\alpha=\frac{1}{2nm}$, the testing rate $\rho^2=1-\frac{1}{3n}$ satisfies \eqref{eq:FinitoCon1c}. This testing rate also satisfies \eqref{eq:FinitoCon2c} if $n\ge 11$. Moreover, this testing rate also satisfies \eqref{eq:FinitoCon3c} under the further assumption $n\ge \frac{48L^2}{m^2}$. Due to the fact $L\ge m$, we always have $n\ge 11$ when $n\ge \frac{48L^2}{m^2}$.  Therefore, Statement 3 directly follows from Corollary \ref{cor:Finito2}.
\end{enumerate}
Now the proof is complete.

\subsection{Analytical Proof for SDCA (Corollary \ref{cor:SDCA})}
To prove Statement 1 in Corollary \ref{cor:SDCA},  we set $\gamma=0$ to reflect the assumption $f_i\in \mathcal{F}(0,L)$. When $\alpha\le\frac{2}{L+2mn}$, we have $\tilde{\alpha}=\alpha m n\le \frac{2mn}{L+2mn}<1$.
 Given the testing rate $\rho^2=1-m\alpha=1-\frac{\tilde{\alpha}}{n}$, it is straightforward to verify  $0\le \rho^2\le 1$ when $\alpha\le \frac{2}{L+2mn}$. 
For this particular rate, the coupled LMI conditions \eqref{eq:SDCALMI1} and \eqref{eq:SDCALMI2} in Statement~2 of Theorem~\ref{thm:LMISector} are feasible with $p_1=\frac{1}{\tilde{\alpha}}$, $p_2=\frac{1-\tilde{\alpha}}{\tilde{\alpha}^2}$, $\lambda_1=0$, and $\lambda_2=\frac{(1-\tilde{\alpha})mn}{\tilde{\alpha}L}$. To see this, first notice $p_2>0$ and $0<\lambda_2\le \frac{1}{2}$ given the fact $\tilde{\alpha}\le \frac{2mn}{L+2mn}<1$. With the given rate $\rho^2=1-\frac{\tilde{\alpha}}{n}$ and the current choice of $(p_1, p_2, \lambda_1, \lambda_2)$, LMIs \eqref{eq:SDCALMI1} and \eqref{eq:SDCALMI2} become
\begin{align}
\bmtx \frac{n}{\tilde{\alpha}}(1-\rho^2)-1 & 0 \\  0 & 1-2\lambda_2 \emtx&=\bmtx 0 & 0 \\ 0 & 1-2\lambda_2\emtx\le 0\\
\bmtx -1-\frac{2n(1-\tilde{\alpha})}{\tilde{\alpha}}+(1-\rho^2)(\frac{n}{\tilde{\alpha}}+\frac{n^2(1-\tilde{\alpha})}{\tilde{\alpha}^2}) & 0\\ 0 & 1-2\lambda_2 \emtx &=\bmtx n(1-\frac{1}{\tilde{\alpha}}) & 0 \\ 0 & 1-2\lambda_2\emtx \le 0
\end{align}
The above LMIs hold due to the fact $\lambda_2\le \frac{1}{2}$ and $\tilde{\alpha}<1$.  
By Theorem \ref{thm:LMISector}, \eqref{eq:SDCArateSC1} holds.

To prove Statement 2 in Corollary \ref{cor:SDCA},  we set $\gamma=0$ to reflect the assumption $f_i\in \mathcal{F}(0,L)$. When $\alpha\le\frac{m}{L^2+m^2n}$, we have $\tilde{\alpha}=\alpha m n\le \frac{m^2 n}{L^2+m^2n}<1$.
 Given the testing rate $\rho^2=1-m\alpha=1-\frac{\tilde{\alpha}}{n}$, it is straightforward to verify  $0\le \rho^2\le 1$ when $\alpha\le \frac{m}{L^2+m^2n}$. 
For this particular rate, the coupled LMI conditions \eqref{eq:SDCALMI1} and \eqref{eq:SDCALMI2} in Statement~2 of Theorem~\ref{thm:LMISector} are feasible with $p_1=\frac{1}{\tilde{\alpha}}$, $p_2=\frac{1-\tilde{\alpha}}{\tilde{\alpha}^2}$, $\lambda_1=\frac{(1-\tilde{\alpha})mn}{\tilde{\alpha}L}$, and $\lambda_2=\frac{1}{2}$. With the given rate $\rho^2=1-\frac{\tilde{\alpha}}{n}$ and the current choice of $(p_1, p_2, \lambda_1, \lambda_2)$, the left side of \eqref{eq:SDCALMI1} becomes a zero matrix and clearly \eqref{eq:SDCALMI1} holds. In addition, \eqref{eq:SDCALMI2} becomes
\begin{align}
\bmtx n(1-\frac{1}{\tilde{\alpha}})+\frac{L^2}{m^2} & 0 \\ 0 & -2\lambda_1 \emtx \le 0
\end{align}
The above inequality holds since we have $\tilde{\alpha}\le \frac{m^2 n}{L^2+m^2 n}$. 
By Theorem \ref{thm:LMISector}, we can conclude that Statement 2 is true.

\section{Conclusion and Future Work}
\label{sec:ConR}

In this paper, we developed a unified routine for analysis of stochastic optimization methods and demonstrate the utility of our proposed routine by analyzing SAGA, Finito, and SDCA under various conditions (with or without individual convexity, etc). Our routine includes five steps:
\begin{enumerate}
\item Choose proper $(A_i, B_i, C)$ to rewrite the stochastic optimization method  as a special case of our general jump system model \eqref{eq:spfdJump0}.
\item Apply Theorem \ref{thm:lmiS0} to obtain an LMI testing condition for the linear convergence rate analysis. 
\item Test LMI \eqref{eq:lmiIQCS0} numerically to  narrow down Lyapunov function structures and useful function inequalities required by the further analysis.
\item Apply linear algebra tricks to convert LMI \eqref{eq:lmiIQCS0} into some equivalent small LMIs whose size do not depend on $n$. 
\item Construct analytical proofs for linear convergence rate bounds using the resultant small LMIs. 
\end{enumerate}

The first step is case-dependent. However, this step is usually straightforward and technically not difficult.
The second and third steps are completely automated and require no tricks at all. 
These two steps can even be done for non-uniform sampling strategy if we slightly modify the LMI condition in Theorem~\ref{thm:lmiS0}. In principle, one can implement \eqref{eq:lmiIQCS0} once, and just needs to update $(\tilde{A}_i, \tilde{B}_i, \tilde{C})$ matrices given any new method.
The fourth step is case-dependent but only requires very basic linear algebra tricks. As long as the assumptions on $f_i$ are the same for all $i$ and a uniform sampling is used, one should be able to obtain such equivalent small LMIs. The fifth step is the most technical step. This step is case-dependent and can be non-trivial for some complicated algorithms, e.g. Finito. However, at least one can numerically solve the resultant small LMIs using semidefinite programming solvers and use the numerical results to guide the analytical proofs.

In the third step, one may realize that LMI \eqref{eq:lmiIQCS0} is not sufficient for analysis of certain methods, e.g. SAG.
Then one needs to exploit more advanced function properties and incorporate more advanced quadratic constraints into the LMI formulations. See \citet{Lessard2014} and \citet{BinHu2017} for detailed discussions on  weighted off-by-one IQC and other advanced quadratic constraints. The applications of these advanced quadratic constraints require much heavier mathematical notation.
 A detailed analysis of more complicated stochastic methods using such advanced quadratic constraints is beyond the scope of this paper, and will be pursued in future research.

We believe our work is just a starting point for further studies of empirical risk minimization using tools from control theory.
 We briefly comment on several possible extensions of our proposed framework to conclude the paper.

\vspace{0.1in}
\textbf{Non-uniform sampling strategy}: Theorem \ref{thm:lmiS0} can be easily modified to handle non-uniform sampling strategy. However, the LMI dimension reduction in this case is non-trivial since the solution for the resultant LMI cannot be easily parameterized using a few scalar decision variables. It requires more efforts to investigate how to reduce the dimension of the resultant LMI in this case. A possible solution may involve properly scaling Lyapunov functions with the sampling distribution.

\vspace{0.1in}

\textbf{Stochastic quadratic constraints and SVRG}: 
SVRG \citep{johnson2013} is an important method which cannot be represented by our jump system model \eqref{eq:spfdJump0}.
The main issue is that SVRG has a deterministic periodic component which cannot be captured by a jump system model. One needs to take the periodicity and the randomness into accounts simultaneously. 
It will be interesting to develop an LMI-based approach for automated analysis and design of SVRG and its non-convex variants \citep{allen2016}. One possible idea is to absorb the randomness and the periodicity into an uncertainty block whose input/output behavior can be characterized by some stochastic quadratic constraints. Similar ideas have already been used to recover the standard convergence results of the SG method  \citep[Chapter 6]{Binthesis}.

\vspace{0.1in}

\textbf{Automated design procedure of stochastic optimization methods}: One may apply our proposed LMIs to numerically design stochastic optimization methods for practical problems. A direct design approach relies on grid search and is similar to the design procedure in \citet[Section 6]{Lessard2014}. A more general design approach may be developed using the following sparse optimization formulation. Based on our general model \eqref{eq:spfdJump0}, a stochastic method is typically characterized by the matrices $(A_i, B_i, C)$. Hence, the design of stochastic methods can be formulated as a sparse optimization problem where we need to select $(A_i, C)$ and sparse $B_i$ for $i=1, \ldots, n$ to minimize the convergence rate $\rho$ under the LMI constraint \eqref{eq:lmiIQCS0} and some other structure constraints. The sparsity of $B_i$ is important since it ensures the per-iteration cost of the resultant method to be low.

\vspace{0.1in}

\textbf{Larger family of non-convex functions}: Notice the main assumption in this paper is $g\in \mathcal{S}(m, L)$, and the convexity of $g$ is not required. There exist convergence results for other families of non-convex functions, e.g. functions satisfying  Polyak-Lojasiewicz (PL) inequality \citep{karimi2016linear, reddi2016stochastic, reddi2016fast}. It is interesting to investigate how to extend our quadratic constraint approach for more general non-convex functions.

\vspace{0.1in}

\textbf{Accelerated methods}: Various acceleration techniques \citep{nitanda2014stochastic,lin2015universal, shalev2016acc, defazio2016simple}  have been proposed to improve the convergence guarantees of the stochastic optimization methods when the big data condition is not met. We will extend our LMI method to analyze stochastic accelerated methods (with or without individual convexity) in the future.

\vspace{0.1in}

\textbf{Randomly-Permuted ADMM with multiple blocks}: 
The alternating direction method of multipliers (ADMM) \citep{boyd2011ADMM} is an important distributed optimization algorithm. There are some initial convergence results on ADMM with multiple blocks \citep{hong2012linear, chen2016direct}. The quantification of the mean-square convergence rates of the so-called randomly-permuted ADMM  with multiple blocks \citep{sun2015expected} remains an open topic. IQCs have been successfully applied to analyze ADMM with two blocks \citep{nishihara2015}. The extension of jump system theory for random-permuted ADMM with multiple blocks is an important future task.

\vspace{0.1in}

\textbf{Asynchronous settings}:  
In parallel computing,  the algorithm performance will typically be impacted by the communication delay and memory contention \citep{recht2011hogwild, zhang2014asynchronous}. In this case, it is necessary to assess the robustness of the optimization methods with respect to the delays in the gradient update.
There exist many IQCs for time-varying delays in the controls literature \citep{Kao2012,  Kao2004, Kao2007, Pfifer2015d}. One may apply a scaling trick to tailor these IQCs for convergence rate analysis \citep{hu16}.
Hence the IQC analysis may be extended to study the impacts of time delays on SAG, SAGA, Finito, SDCA and other related stochastic optimization methods.

\acks{The authors would like to thank the anonymous reviewers
for their constructive comments. Bin Hu and Peter Seiler were supported by the National Science Foundation under Grant
No. NSF-CMMI-1254129 entitled “CAREER: Probabilistic Tools for High
Reliability Monitoring and Control of Wind Farms.”  Bin Hu and Peter Seiler were also
supported by the NASA Langley NRA Cooperative Agreement NNX12AM55A
entitled “Analytical Validation Tools for Safety Critical Systems
Under Loss-ofControl Conditions”, Dr. Christine Belcastro technical
monitor. Anders Rantzer is a member of the LCCC Linnaeus Center and the ELLIIT Excellence Center at Lund University. His contribution was supported by the Swedish Research Council, grant 2016-04764, and the Institute for Mathematics and its Applications at University of Minnesota. }



\bibliography{IQCandSOS}

\appendix

\section{Jump System Formulations of SAG, Finito, and SDCA}
\label{sec:JumpF}
\begin{enumerate}
\item (SAG):  Define $w^k=\bmtx \nabla f_1(x^k)^T \cdots \nabla f_n(x^k)^T \emtx^T$, and then the SAG gradient update rule~\eqref{eq:SAGy} can still be rewritten as \eqref{eq:SAGyJump}. Notice $\sum_{i=1}^n y_i^{k} = ( e^T \otimes I_p) y^{k}$ and $\nabla f_{i_k}(x^k)-y_{i_k}^k=(e_{i_k}^T\otimes I_p)(w^k-y^k)$.
Thus the iteration rule \eqref{eq:SAGx} can be rewritten as follows:
\begin{align}
\label{eq:SAGxJump}
\begin{split}
x^{k+1}&=x^k-\alpha\left(\frac{\nabla f_{i_k}(x^k)-y_{i_k}^k}{n}+\frac{1}{n}\sum_{i=1}^n y_i^k\right)\\
&=x^k-\frac{\alpha}{n}(e_{i_k}^T\otimes I_p) (w^k-y^k)-\frac{\alpha}{n}(e^T\otimes I_p) y^k\\
&=x^k-\frac{\alpha}{n}\left((e-e_{i_k})^T\otimes I_p\right)y^{k}-\frac{\alpha}{n}(e_{i_k}^T\otimes I_p) w^k
\end{split}
\end{align}
At this point, both the gradient update in \eqref{eq:SAGyJump} and the iteration update in \eqref{eq:SAGxJump} depend on $w^k=\bmtx \nabla f_1(x^k)^T \cdots \nabla f_n(x^k)^T \emtx^T$.  The key step in the modeling is
to "separate  out" this nonlinear term. Setting $v^k=x^k$ and then $w^k=\bmtx \nabla f_1(v^k)^T \cdots \nabla f_n(v^k)^T \emtx^T$. Now the update rules in \eqref{eq:SAGyJump} and \eqref{eq:SAGxJump} can be expressed as:
\begin{align}
\label{eq:SAGG}
\begin{split}
\bmtx y^{k+1} \\ x^{k+1}\emtx&=\bmtx (I_n-e_{i_k} e_{i_k}^T)\otimes I_p & \tilde{0} \otimes I_p \\ -\frac{\alpha}{n}(e-e_{i_k})^T\otimes I_p & I_p \emtx \bmtx y^{k} \\ x^k \emtx+\bmtx(e_{i_k} e_{i_k}^T)\otimes I_p \\ (-\frac{\alpha}{n}e_{i_k}^T)\otimes I_p \emtx w^k\\
v^k&=\bmtx \tilde{0}^T \otimes I_p& I_p \emtx \bmtx y^{k}\\ x^k\emtx\\
w^k&=\bmtx \nabla f_1(v^k) \\ \vdots \\ \nabla f_n(v^k) \emtx
\end{split}
\end{align}
which is exactly in the form of the general jump system model \eqref{eq:spfdJump0} with $\xi^k = \bsmtx y^k \\  x^k\esmtx$. Recall that $w^*=\bmtx \nabla f_1(x^*)^T &\ldots &\nabla f_n(x^*)^T\emtx^T$. It is trivial to set $\xi^*=\bmtx (w^*)^T & (x^*)^T\emtx^T$, and verify that \eqref{eq:spfdJumpEqi} holds. 
\item (Finito): Recall that we denote $y^k=\bmtx (y_1^k)^T  & \cdots  & (y_n^k)^T \emtx^T$ and $x^k=\bmtx (x_1^k)^T  & \cdots  & (x_n^k)^T \emtx^T$. We set $v^k$ as
\begin{align}
\label{eq:Finitov} 
v^k= \frac{1}{n}\sum_{i=1}^n x_i^k-\alpha\sum_{i=1}^n  y_i^k
\end{align}
Again, we set $w^k=\bmtx \nabla f_1(v^k)^T \cdots \nabla f_n(v^k)^T \emtx^T$. Then we can immediately rewrite \eqref{eq:Finitoy} as
\begin{align}
\label{eq:FinitoyJ}
y^{k+1}=\left((I_n-e_{i_k} e_{i_k}^T)\otimes I_p\right) y^{k}+ \left((e_{i_k} e_{i_k}^T)\otimes I_p\right) w^k
\end{align}
It is also straightforward to rewrite \eqref{eq:Finitox} as
\begin{align}
\label{eq:FinitoxJ}
x^{k+1}=\left((I_n-e_{i_k} e_{i_k}^T+\frac{1}{n}(e_{i_k}e^T))\otimes I_p\right) x^{k}-\alpha\left((e_{i_k} e^T)\otimes I_p\right) y^k
\end{align}
Therefore, we can combine \eqref{eq:Finitov}, \eqref{eq:FinitoyJ}, and \eqref{eq:FinitoxJ} to obtain
\begin{align}
\begin{split}
\bmtx y^{k+1} \\ x^{k+1}\emtx&=\bmtx (I_n- e_{i_k} e_{i_k}^T)\otimes I_p & (\tilde{0}\tilde{0}^T) \otimes I_p \\  -\alpha (e_{i_k}e^T)\otimes I_p &  (I_n- e_{i_k} e_{i_k}^T+\frac{1}{n}(e_{i_k}e^T))\otimes I_p \emtx \bmtx y^{k} \\ x^{k} \emtx+\bmtx( e_{i_k} e_{i_k}^T)\otimes I_p \\ (\tilde{0}\tilde{0}^T)\otimes I_p \emtx w^k\\
v^k&=\bmtx  -\alpha e^T\otimes I_p  & \frac{1}{n}e^T\otimes I_p \emtx \bmtx y^{k}\\ x^k\emtx\\
w^k&=\bmtx \nabla f_1(v^k) \\ \vdots \\ \nabla f_n(v^k) \emtx
\end{split}
\end{align}
which is exactly in the form of the general jump system model \eqref{eq:spfdJump0} with $\xi^k = \bsmtx y^k \\  x^k\esmtx$.

Notice $\xi^k\in \R^{(n+1)p}$ for SAG and SAGA, but $\xi^k\in \R^{2np}$ for Finito. Hence in general,  Finito requires more memory compared with SAG and SAGA. Based on the fact $\sum_{i=1}^n \nabla f_i(x^*)=0$, we can set $\xi^*=\bmtx w^* \\ e \otimes x^* \emtx$, and verify that \eqref{eq:spfdJumpEqi}  holds. Therefore, if $\xi^k$ converges to $\xi^*$, then $y_i^k$ converges to $\nabla f_i(x^*)$ and $x_i^k$ converges to $x^*$.

\item (SDCA): We still have $y^k=\bmtx (y_1^k)^T  & \cdots  & (y_n^k)^T \emtx^T$. The update rule \eqref{eq:SDCAx} can be rewritten as
\begin{align}
\label{eq:SDCAxJ}
x^k=\frac{1}{m n}(e^T\otimes I_p) y^k
\end{align}
Again, $w^k=\bmtx \nabla f_1(v^k)^T \cdots \nabla f_n(v^k)^T \emtx^T$. Hence we can set $v^k=x^k$ and rewrite the update rule \eqref{eq:SDCAy} as
\begin{align}
\label{eq:SDCAyJ}
y^{k+1}=\left((I_n-\alpha m n e_{i_k} e_{i_k}^T)\otimes I_p\right) y^{k}-\alpha mn \left((e_{i_k} e_{i_k}^T)\otimes I_p\right) w^k
\end{align}
We can augment \eqref{eq:SDCAxJ} and \eqref{eq:SDCAyJ} as
\begin{align}
\label{eq:SDCAJ}
\begin{split}
y^{k+1}&=\left((I_n-\alpha m n e_{i_k} e_{i_k}^T)\otimes I_p\right) y^{k}-\alpha m n \left((e_{i_k} e_{i_k}^T)\otimes I_p\right) w^k\\
v^k&=\left(\frac{1}{mn}e^T\otimes I_p \right) y^k\\
w^k&=\bmtx \nabla f_1(v^k) \\ \vdots \\ \nabla f_n(v^k) \emtx
\end{split}
\end{align}
which is exactly in the form of the general jump system model \eqref{eq:spfdJump0} with
$\xi^k=y^k$.  Notice the state $\xi^k$ is completely determined by $y^k$, and does not directly depend on $x^k$.

\end{enumerate}

\section{Numerical Tests Using the LMI Condition in Theorem \ref{thm:lmiS0}}
\label{sec:numerical}
We can numerically solve LMI \eqref{eq:lmiIQCS0} in Theorem \ref{thm:lmiS0} and get some rough ideas of the feasibility of the proposed LMI conditions.

First, we apply the proposed LMI condition to analyze the convergence rate of SAGA. 
The most relevant existing result for this case was presented in \citet[Section 2]{defazio2014} and states the following fact. Under the assumption that $g\in \mathcal{F}(m,L)$ and $f_i\in \mathcal{F}(m,L)$,  the SAGA iteration with the stepsize $\alpha=\frac{1}{3L}$ converges at a linear rate $\rho=\sqrt{1-\min\{\frac{m}{3L},\frac{1}{4n}\}}$ in the mean square sense. 
Therefore, for any $m$, $L$, and $n$, we can choose $\rho=\sqrt{1-\min\{\frac{m}{3L},\frac{1}{4n}\}}$ and numerically test the feasibility of the resultant LMI \eqref{eq:lmiIQCS0} using CVX~\citep{cvx,gb08}
with the solver SDPT3~\citep{sdpt3_03,sdpt3_99}.  As discussed before, we should set $\nu=\gamma=-m$ to reflect the assumptions $g\in \mathcal{F}(m,L)$ and $f_i\in \mathcal{F}(m,L)$. A practical issue is that the LMI is homogeneous, i.e. if $(\tilde{P},\lambda_1,\lambda_2)$ is a feasible solution then $(c \tilde{P},c\lambda_1,c\lambda_2)$ is also a feasible solution for any $c>0$.  This homogeneity can cause numerical issues.  One method to break this homogeneity is to replace $\tilde{P}>0$ with the condition $\tilde{P} \ge 10^{-2} I$.
 Based on some preliminary feasibility tests with relatively small $n$ ($n<100$), the proposed LMI remains feasible  even if  the following simple parameterization of $\tilde{P}$ is used 
\begin{align}
\label{eq:SAGAP}
\tilde{P}=\bmtx p_1 I_n  & \tilde{0} \\ \tilde{0}^T & p_2\emtx
\end{align}
We notice that LMI \eqref{eq:lmiIQCS0} seems always feasible with the choice of  $\rho=\sqrt{1-\min\{\frac{m}{3L},\frac{1}{4n}\}}$. This numerically confirms the existing rate result for $n$ being up to several hundred. We further notice that the LMI can be feasible with $\rho^2$ smaller than $1-\min\{\frac{m}{3L},\frac{1}{4n}\}$. This indicates that one may get sharper rate bounds for SAGA using our proposed LMI. Finally, treating $\tilde{P}$ as an unknown matrix or parameterizing $\tilde{P}$ as \eqref{eq:SAGAP} often does not change the feasibility of the resultant LMI. This implies that adopting the parameterization \eqref{eq:SAGAP} does not introduce further conservatism into our analysis.

Similar testing can also be performed if $f_i$ is only assumed to be $L$-smooth.  We only need to modify the value of $\gamma$ to be $L$. The numerical results suggest that using a simple parameterization \eqref{eq:SAGAP} does not introduce further conservatism in this case.
We can also perform such naive numerical analysis for SDCA, Finito and SAG for $n$ being up to several hundred. The numerical results obtained by the proposed semidefinite programs actually inspire our analytical proofs for SAGA, SDCA, and Finito.

\section{Proof of Theorem~\ref{thm:LMISector}}
\label{sec:ProofDR}
The proof is based on the following key  linear algebra result which can be used to transform certain high dimensional LMIs into two much smaller coupled LMIs.

\begin{lemma}
\label{lem:KeyLA}
The following statements are true:
\begin{enumerate}
\item  $\mu_{1} I_n +q_{1} ee^T>0$ if and only if $\mu_{1} > 0$ and $\mu_{1}+n q_{1} >0$.
\item 
\begin{align}
\bmtx \mu_{1} I_n +q_1 ee^T & \mu_3 I_n +q_3 ee^T \\ \mu_3 I_n +q_3 ee^T & \mu_2 I_n+q_2 ee^T\emtx\le 0
\end{align}
if and only if
\begin{align}
&\bmtx \mu_1 & \mu_3 \\ \mu_3 & \mu_2\emtx \le 0, \\
\bmtx \mu_1 & \mu_3 \\ \mu_3 & \mu_2\emtx& + n \bmtx q_1 & q_3 \\ q_3 & q_2 \emtx  \le 0
\end{align}
\item 
\begin{align}
\bmtx \mu_{1} I_n +q_1 ee^T & q_4 e & \mu_6 I_n+q_6 ee^T\\ q_4 e^T & \mu_2  &  q_5 e^T \\ \mu_6 I_n+q_6 ee^T & q_5 e & \mu_3 I_n+ q_3 ee^T\emtx\le 0
\end{align}
if and only if
\begin{align}
\bmtx \mu_1 & 0 & \mu_6 \\ 0 & \mu_2 & 0 \\ \mu_6 & 0 & \mu_3 \emtx \le 0, \\
\bmtx \mu_1 +n q_1& \sqrt{n}q_4 & \mu_6 + nq_6 \\ \sqrt{n}\mu_4 & \mu_2 & \sqrt{n}q_5 \\ \mu_6+n q_6 & \sqrt{n}q_5 & \mu_3+nq_3 \emtx&   \le 0
\end{align}
\item 
\begin{align}
\bmtx \mu_{1} I_n +q_1 ee^T & \mu_4 I_n +q_4 ee^T & \mu_6 I_n+q_6 ee^T\\ \mu_4 I_n +q_4 ee^T & \mu_2 I_n+q_2 ee^T & \mu_5 I_n+q_5 ee^T \\ \mu_6 I_n+q_6 ee^T & \mu_5 I_n +q_5 ee^T & \mu_3 I_n+ q_3 ee^T\emtx\le 0
\end{align}
if and only if
\begin{align}
&\bmtx \mu_1 & \mu_4 & \mu_6 \\ \mu_4 & \mu_2 & \mu_5 \\ \mu_6 & \mu_5 & \mu_3 \emtx \le 0, \\
\bmtx \mu_1 & \mu_4 & \mu_6  \\ \mu_4 & \mu_2 & \mu_5 \\ \mu_6 & \mu_5 & \mu_3 \emtx& + n \bmtx q_1 & q_4 & q_6  \\ q_4 & q_2 & q_5 \\ q_6 & q_5 & q_3 \emtx  \le 0
\end{align}
\end{enumerate} 
\end{lemma}
\begin{proof}
Let $Q\in \R^{n\times (n-1)}$ be a matrix such that $\bmtx \frac{e}{\sqrt{n}} & Q\emtx$ is orthogonal. Then
\begin{align}
\bmtx \frac{e}{\sqrt{n}} & Q\emtx^T (\mu_1 I_n+q_1 ee^T) \bmtx \frac{e}{\sqrt{n}} & Q\emtx= \mbox{diag}(\mu_1+n q_1, \mu_1, \ldots, \mu_1)
\end{align}
Statement 1 directly follows since $\bmtx \frac{e}{\sqrt{n}} & Q\emtx$ is invertible. 
Similarly, Statement 2 can be immediately proved using the following fact:
\begin{align}
 &\bmtx \frac{e}{\sqrt{n}} & \tilde{0} & Q & 0Q\\ \tilde{0} & \frac{e}{\sqrt{n}} & 0Q & Q\emtx^T  \bmtx \mu_{1} I_n +q_1 ee^T & \mu_3 I_n +q_3 ee^T \\ \mu_3 I_n +q_3 ee^T & \mu_2 I_n+q_2 ee^T\emtx \bmtx \frac{e}{\sqrt{n}} & \tilde{0} & Q & 0Q\\ \tilde{0} & \frac{e}{\sqrt{n}} & 0Q & Q\emtx\\
= &\mbox{diag}\left(\bmtx \mu_1+n q_1 & \mu_3+nq_3\\ \mu_3+nq_3 & \mu_2+nq_2\emtx, \bmtx \mu_1 & \mu_3 \\ \mu_3 & \mu_2 \emtx \otimes I_{n-1}\right)
\end{align}
Statement 4 can be proved using a similar argument.
Finally, Statement 3 can be proved using Statement 2 and a Schur complement argument.
\end{proof}
When analyzing SDCA, we can apply Statement 2 of the above lemma to convert LMI \eqref{eq:lmiIQCS0} into two coupled $2\times 2$ LMIs whose feasibility can be checked analytically. Similarly,
Statement 3 of the above lemma  is useful  for the rate analysis of SAGA, and Statement 4 of the above lemma is useful for the rate analysis of Finito. Now we only need to substitute $(\tilde{A}_i, \tilde{B}_i, \tilde{C})$ and $\tilde{P}$ into the left side of \eqref{eq:lmiIQCS0}, and then Theorem \ref{thm:LMISector} directly follows from the above lemma.

\begin{enumerate}
\item To prove Statement 1 of Theorem \ref{thm:LMISector}, recall that we have
$\tilde{P}=\bmtx p_1 I_n & \tilde{0} \\ \tilde{0}^T & p_2\emtx$.
For SAGA, it is straightforward to verify
\begin{align}
\frac{1}{n}\sum_{i=1}^n\tilde{A}_i\tilde{P} \tilde{A}_i&=\bmtx (\frac{p_2 \alpha^2}{n}+\frac{n-1}{n}p_1)I_n -\frac{\alpha^2 p_2 }{n^2}ee^T& \tilde{0}\\ \tilde{0}^T & p_2 \emtx\\
\frac{1}{n}\sum_{i=1}^n\tilde{B}_i\tilde{P} \tilde{A}_i&=\bmtx -\frac{\alpha^2 p_2}{n}I_n+\frac{\alpha^2 p_2}{n^2}ee^T \\[2mm] -\frac{\alpha p_2}{n} e^T \emtx\\
\frac{1}{n}\sum_{i=1}^n\tilde{B}_i\tilde{P} \tilde{B}_i&=\frac{p_1+\alpha^2 p_2}{n}I_n
\end{align}

In addition, we have 
\begin{align}
\begin{split}
 \bmtx  \tilde{C}^T\tilde{D}_{\psi 1}^T \\ \tilde{D}_{\psi 2}^T \emtx  &\left(\bmtx \lambda_1 & \tilde{0}^T \\ \tilde{0} & \frac{\lambda_2}{n} I_n \emtx \otimes  \bmtx 0 & 1\\ 1 & 0\emtx\right)
  \bmtx \tilde{D}_{\psi 1}\tilde{C} & \tilde{D}_{\psi2} \emtx =\\
&\lambda_1 \bmtx 0_n & \tilde{0} & 0_n \\ \tilde{0}^T & -2mL & \frac{m+L}{n} e^T \\ 0_n & \frac{m+L}{n}e & -\frac{2}{n^2}ee^T \emtx +\lambda_2 \bmtx 0_n & \tilde{0} & 0_n \\ \tilde{0}^T & 2L\gamma & \frac{L-\gamma}{n} e^T \\ 0_n & \frac{L-\gamma}{n}e & -\frac{2}{n}I_n \emtx
\end{split}
  \end{align}

Now we can directly prove Statement 1 of Theorem \ref{thm:LMISector} by applying Statement 3 of Lemma \ref{lem:KeyLA} to convert \eqref{eq:lmiIQCS0} into small coupled LMIs.

\item
To prove Statement 2 of  Theorem \ref{thm:LMISector},  recall that we have
\begin{align}
\tilde{P}=\bmtx p_1 I_n+p_2 ee^T & p_3 ee^T \\ p_3 ee^T & p_4 I_n+ p_5 ee^T\emtx 
\end{align}

Hence it is straightforward to verify:
\begin{align}
&\frac{1}{n}\sum_{i=1}^n\tilde{A}_i\tilde{P} \tilde{A}_i=\bmtx W_{11} & W_{12}\\ W_{12}^T & W_{22} \emtx\\
&\frac{1}{n}\sum_{i=1}^n\tilde{B}_i\tilde{P} \tilde{A}_i=\bmtx -\frac{p_2}{n}I_n+\frac{1}{n}(p_2-  p_3\alpha)ee^T \\[3mm] -\frac{p_3}{n}I_n+\frac{(n+1)p_3}{n^2} ee^T \emtx\\
&\frac{1}{n}\sum_{i=1}^n\tilde{B}_i\tilde{P} \tilde{B}_i=\frac{p_1+p_2}{n}I_n
\end{align}
where $W_{11}$, $W_{12}$ and $W_{22}$ are computed as
\begin{align}
W_{11}&= \left(\frac{p_2}{n}+\frac{n-1}{n}p_1\right)I_n+\left((1-\frac{2}{n})p_2-2(1-n^{-1})p_3\alpha+(p_4+p_5)\alpha^2 \right)ee^T\\
W_{12}&=\frac{p_3}{n} I_n+\frac{(n-1-n^{-1}) p_3-p_4\alpha-n p_5\alpha}{n} ee^T\\
W_{22}&=\left(\frac{p_5}{n}+(1-\frac{1}{n})p_4\right)I_n+\left(\frac{p_4}{n^2}+(1-\frac{1}{n^2})p_5\right)ee^T 
\end{align}
Then we can combine Statement 4 of Lemma \ref{lem:KeyLA} with the following formula to prove Statement~2 of Theorem \ref{thm:LMISector}.
\end{enumerate}
\begin{align}
\begin{split}
  &\bmtx  \tilde{C}^T\tilde{D}_{\psi 1}^T \\ \tilde{D}_{\psi 2}^T \emtx  \left(\bmtx \lambda_1 & \tilde{0}^T \\ \tilde{0} & \frac{\lambda_2}{n} I_n \emtx \otimes  \bmtx 0 & 1\\ 1 & 0\emtx\right)
  \bmtx \tilde{D}_{\psi 1}\tilde{C} & \tilde{D}_{\psi2} \emtx =\\
  & \lambda_1 \bmtx -2Lm\alpha^2ee^T & \frac{2Lm\alpha}{n}ee^T & -\frac{(m+L)\alpha}{n} ee^T \\[2mm] \frac{2Lm\alpha}{ n}ee^T & -\frac{2mL}{n^2}ee^T & \frac{L+m}{n^2} ee^T \\[2mm]   -\frac{(m+L)\alpha}{n}ee^T &  \frac{L+m}{n^2}ee^T & -\frac{2}{n^2}ee^T \emtx+\lambda_2 \bmtx 2L\gamma\alpha^2ee^T & -\frac{2L\gamma\alpha}{n}ee^T & -\frac{(L-\gamma)\alpha}{n} ee^T \\[2mm] -\frac{2L\gamma\alpha}{n}ee^T & \frac{2L\gamma}{n^2}ee^T & \frac{L-\gamma}{n^2} ee^T \\[2mm]   -\frac{(L-\gamma)\alpha}{n}ee^T &  \frac{L-\gamma}{n^2}ee^T & -\frac{2}{n}I_n \emtx
\end{split}
  \end{align}

\begin{enumerate} \addtocounter{enumi}{2}
\item To prove Statement 3 of Theorem \ref{thm:LMISector},   we have  $\tilde{P}=p_1 I_n +p_2 ee^T$ and $\tilde{\alpha}=\alpha m n$.
Hence it is straightforward to obtain the following formulas:

\begin{align}
\frac{1}{n}\sum_{i=1}^n\tilde{A}_i\tilde{P} \tilde{A}_i&=\left(\frac{p_1(\tilde{\alpha}^2-2\tilde{\alpha}+n)}{n}+\frac{p_2\tilde{\alpha}^2}{n}\right)I_n-\frac{p_2(2\tilde{\alpha}-n)}{n}ee^T\\
\frac{1}{n}\sum_{i=1}^n\tilde{B}_i\tilde{P} \tilde{A}_i&=\left(\frac{p_1(\tilde{\alpha}^2-\tilde{\alpha})}{n}+\frac{p_2\tilde{\alpha}^2}{n}\right)I_n-\frac{\tilde{\alpha} p_2}{n}ee^T\\
\frac{1}{n}\sum_{i=1}^n\tilde{B}_i\tilde{P} \tilde{B}_i&=\frac{(p_1+ p_2)\tilde{\alpha}^2}{n}I_n
\end{align}

In addition,  we can directly obtain
\begin{align}
\begin{split}
 \bmtx  \tilde{C}^T\tilde{D}_{\psi 1}^T \\ \tilde{D}_{\psi 2}^T \emtx  &\left(\bmtx \lambda_1 & \tilde{0}^T \\ \tilde{0} & \frac{\lambda_2}{n} I_n \emtx \otimes  \bmtx 0 & 1\\ 1 & 0\emtx\right)
  \bmtx \tilde{D}_{\psi 1}\tilde{C} & \tilde{D}_{\psi2} \emtx =\\
&\lambda_1 \bmtx 0_n & \frac{L}{mn^2}ee^T \\[2mm] \frac{L}{mn^2}ee^T & -\frac{2}{n^2}ee^T \emtx +\lambda_2 \bmtx  \frac{2L\gamma}{m^2n^2}ee^T & \frac{L-\gamma}{mn^2}e e^T \\[2mm]  \frac{L-\gamma}{mn^2}e e^T  & -\frac{2}{n}I_n \emtx
\end{split}  
\end{align}
Now Statement 3 of Theorem \ref{thm:LMISector} directly follows from Statement 2 of Lemma \ref{lem:KeyLA}.
\end{enumerate}
Now the proof of Theorem \ref{thm:LMISector} is complete.

%
%
%

\end{document}